\documentclass{article}

\usepackage{microtype}
\usepackage{graphicx}
\usepackage{booktabs} 

\usepackage{hyperref}

\usepackage[accepted]{icml2024}

\usepackage{amsmath}
\usepackage{amssymb}

\usepackage{mathtools}
\usepackage{amsthm}
\usepackage{caption}
\usepackage{subcaption}
\usepackage{multirow}

\usepackage{url}            
\usepackage{amsfonts}       
\usepackage{nicefrac}       
\usepackage{microtype}      
\usepackage{xcolor}         
\usepackage{xcolor}
\usepackage{bm}
\usepackage{enumitem}
\usepackage{multirow}

\usepackage[toc,page,header]{appendix}
\usepackage{minitoc}

\usepackage[capitalize,noabbrev]{cleveref}

\theoremstyle{plain}
\newtheorem{theorem}{Theorem}[section]
\newtheorem{proposition}[theorem]{Proposition}
\newtheorem{lemma}[theorem]{Lemma}

\theoremstyle{definition}
\newtheorem{definition}[theorem]{Definition}

\theoremstyle{remark}

\usepackage[textsize=tiny]{todonotes}

\def\NN{{\mathbb N}}    
\def\RR{{\mathbb R}}    
\def\PP{{\mathbb P}}     
\def\KK{{\mathbb K}}     
\def\EE{{\mathbb E}}    
\def\11{{\mathbf 1}}    

  \def\cG{{\mathcal G}}     \def\cH{{\mathcal H}}                    \def\cL{{\mathcal L}} \def\cR{{\mathcal R}} \def\cX{{\mathcal X}} \def\cY{{\mathcal Y}}  \def\cZ{{\mathcal Z}}

\def\mfA{{\mathfrak A}} \def\mfA{{\mathfrak P}}

\def\d{\,{\mathrm d}}

\def\NN{{\mathbb N}}    
\def\RR{{\mathbb R}}    
\def\PP{{\mathbb P}}     
\def\KK{{\mathbb K}}     
\def\EE{{\mathbb E}}    
\def\11{{\mathbf 1}}    

  \def\cG{{\mathcal G}}     \def\cH{{\mathcal H}}                    \def\cL{{\mathcal L}} \def\cR{{\mathcal R}} \def\cX{{\mathcal X}} \def\cY{{\mathcal Y}}  \def\cZ{{\mathcal Z}}

\def\mfA{{\mathfrak A}} \def\mfA{{\mathfrak P}}

\def\d{\,{\mathrm d}}

\definecolor{applegreen}{rgb}{0.01, 0.75, 0.24}

\icmltitlerunning{Domain Generalisation via Imprecise Learning}

\begin{document}

\twocolumn[
\icmltitle{Domain Generalisation via Imprecise Learning}



\icmlsetsymbol{equal}{*}

\begin{icmlauthorlist}
\icmlauthor{Anurag Singh}{yyy}
\icmlauthor{Siu Lun Chau}{yyy}
\icmlauthor{Shahine Bouabid}{sch}
\icmlauthor{Krikamol Muandet}{yyy}
\end{icmlauthorlist}

\icmlaffiliation{yyy}{CISPA Helmholtz Center for Information Security, Saarbr{\"u}cken, Germany}
\icmlaffiliation{sch}{Department of Statistics, University of Oxford, UK}
\icmlcorrespondingauthor{Anurag Singh}{anurag.singh@cispa.de}

\icmlkeywords{Machine Learning, ICML}

\vskip 0.3in
]



\printAffiliationsAndNotice{Anurag Singh is part of the Graduate School of Computer Science at Saarland University, Saarbr{\"u}cken, Germany.}

\begin{abstract}
Out-of-distribution (OOD) generalisation is challenging because it involves not only learning from empirical data, but also deciding among various notions of generalisation, e.g., optimising the average-case risk, worst-case risk, or interpolations thereof. While this choice should in principle be made by the model operator like medical doctors, this information might not always be available at training time. The institutional separation between machine learners and model operators leads to arbitrary commitments to specific generalisation strategies by machine learners due to these deployment uncertainties. We introduce the Imprecise Domain Generalisation framework to mitigate this, featuring an imprecise risk optimisation that allows learners to stay imprecise by optimising against a continuous spectrum of generalisation strategies during training, and a model framework that allows operators to specify their generalisation preference at deployment. Supported by both theoretical and empirical evidence, our work showcases the benefits of integrating imprecision into domain generalisation.

\end{abstract}

\section{Introduction}

The capability to generalise knowledge, a hallmark of both biological and artificial intelligence (AI), has seen remarkable progress in recent years. Developments in general-purpose learning algorithms \citep{Vapnik91:ERM,HofSchSmo08,Lecun15:DL,Goodfellow16:DL}, model architectures \citep{Krizhevsky12:ImageNet,Cohen16:Group,Vaswani17:Attention}, and training infrastructures \citep{Ratner19:MLsys} have given rise to AI systems such as generative models (GenAI) and large language models (LLM) that surpass human-level generalisation capabilities in specific domains.

Despite notable achievements, these systems may catastrophically fail when operated on out-of-domain (OOD) data because theoretical guarantees for their generalisation hinge on the assumption of independent and identically distributed (IID) training and deployment data, with empirical risk minimisation (ERM) being the dominant learning algorithm \citep{Vapnik91:ERM,Vapnik95:SLT}. 
Emerging challenges like distribution shifts \citep{Candela09:DS,Beery18:WildLife,Beery20:iWildCam,Koh21:WILDS}, adversarial attacks \citep{Szegedy13:Properties,Goodfellow14:Explaining}, and strategic manipulations \citep{Hardt16:Strategic,Perdomo20:Performative,vo_causal_2023} have prompted researchers to question the validity of algorithms developed under this assumption. This gap has fueled interest in OOD generalisation, prompting the exploration of novel learning algorithms and resulting in rapid developments in domain adaptation \citep{Winson20:UDA,Zhao22:UDA}, domain generalisation \citep{Wang21:DG-Review,Zhou21:DG-Review,Shen21:OOD-Review}, and test-time adaptation \citep{Sun20:Test-Time,Wang21:Tent,Chen23:Improved-TTT}, among others.

In IID generalisation, where test loss aligns with training loss, the learner's goal of minimising the training loss aligns with the operator's expectation of small test loss. Bounded data uncertainty, within finite data, enables the learner to assess model generalisation during deployment. 
Historically, the IID assumption is accompanied by another critical, but often overlooked assumption: the overlap between the learner and the operator, who employs the model in real-world contexts.
Conversely, OOD generalisation still lacks a precise definition, leading to additional ambiguity termed ``generalisation uncertainty''. Unlike data uncertainty, generalisation uncertainty arises from a lack of knowledge about deployment environments, whether due to natural shifts (across hospitals, experimental conditions, and time) or artificial ones (adversarial attacks, strategic manipulation), and cannot be mitigated by additional data collection. 

\begin{figure*}
    \centering
    \includegraphics[width=\textwidth]{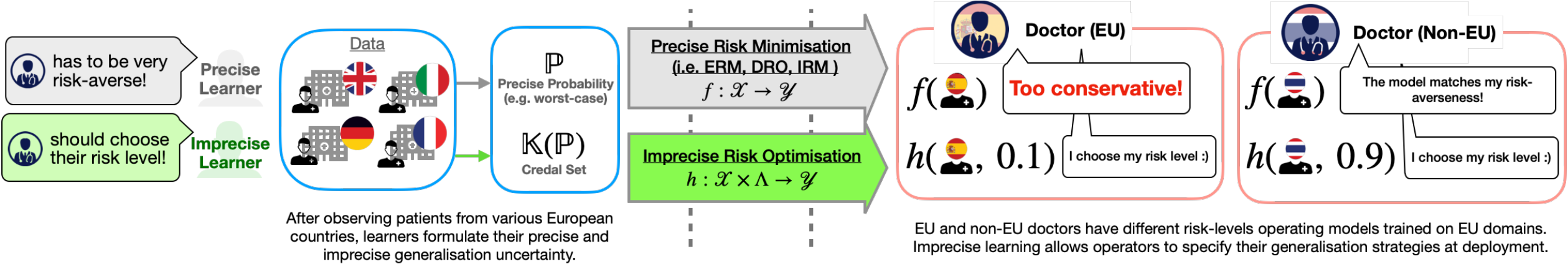}
    \caption{An illustration of our proposed \textit{imprecise learning} framework. We allow learners to stay imprecise to avoid over-commit in light of generalisation uncertainty. Instead, we defer this choice of precise generalisation to the operator.}
    \label{fig:imprecise-learning}
\end{figure*}

Prior research has addressed generalisation uncertainty independently by introducing various concepts of OOD generalisation including worst-case generalisation \citep{Arjovsky19:IRM,Ben-Tal09:RO,Sagawa20:DRO,krueger_out--distribution_2021}, average-case generalisation \citep{Blanchard11:Generalize,Blanchard21:DG-MTL,Muandet13:DG,Zhang21:ARM}, and their interpolations \citep{Eastwood22:QRM}. Learning algorithms like distributional robust optimisation (DRO) \citep{rahimian_distributionally_2022}, invariant risk minimisation (IRM) \citep{Arjovsky19:IRM}, and quantile risk minimisation (QRM) \citep{Eastwood22:QRM} have been tailored for these OOD generalisation notions. This line of research relaxes the IID assumption, but still assumes alignment between the learner's objective and the operator's goal to tackle generalisation uncertainty. Due to the need for precise concept of generalisation in these scenarios, we collectively term them ``precise generalisation''.

Precise generalisation hinges on the assumption that the learner's objective aligns with the operator's goal, necessitating close collaboration during model development. However, this approach presents two primary drawbacks. Firstly, institutional separation between the learner (e.g., machine learning engineers) and model operators (e.g., doctors) can make collaboration costly, time-consuming, or even impractical. Secondly, tailoring the model to a specific operator may restrict its deployment usability, as the operator's beliefs can change or conflicts may emerge when the model is operated by a different individual.
Consider an example depicted in \Cref{fig:imprecise-learning}. Using data obtained from hospitals across Europe, an engineer is developing a machine learning model that will be embedded into a medical software that will be used by the doctors. Here, the engineer confronts uncertainty regarding where the model will ultimately be deployed—it could be within Europe (IID) or outside it (OOD). The engineer might anticipate the doctor's generalisation strategy during the model's training phase. For instance, if the doctor is perceived to be risk-averse, the engineer might prioritise training a model robust to worst-case scenarios. However, ideally, it should be the doctor, often equipped with domain-specific expertise, who decides the generalisation strategy, drawing upon their in-depth knowledge of the field, at deployment time. Customising models effectively to the clinical settings where they operate can significantly impact healthcare outcomes \citep{Beede20:AI}.


In this work, we extend the relaxation of the IID assumption further by loosening the requirement for overlap between the learner and the operator. Since there is no need of specific concept of generalisation at training time, we term this scenario ``imprecise generalisation'' (see \Cref{fig:imprecise-learning}). We operationalise imprecise learning in the context of domain generalisation \citep{Blanchard11:Generalize,Blanchard21:DG-MTL,Muandet13:DG}, aiming to answer the question: \emph{How to take knowledge acquired from an arbitrary number of related domains and apply it to previously unseen domains?} This concept comprises two main components: (1) An optimisation process enabling learners to remain imprecise during learning, thus not committing to a specific generalisation notion during training, and (2) a model framework allowing operators to define their preferred generalisation strategy at deployment. We delve into the formulation and existing work on OOD generalisation in \Cref{sec:preliminaries}. Our primary contribution, the framework of Imprecise Domain Generalisation, is detailed in \Cref{sec:imprecise-dg}, along with its optimisation strategy, termed Imprecise Risk Optimisation (IRO), in \Cref{sec:algorithm}. Experimental results are presented in \Cref{sec:experiments}, and we conclude our paper in \Cref{sec:conclusion}.

All proofs are in the appendix and we open-source our code at \url{https://github.com/muandet-lab/dgil}.

\section{Preliminaries}
\label{sec:preliminaries}

Consider $\cX\subseteq \RR^d$ as our instance space and $\cY$ as our target space, where $\cY \subseteq \RR$ is used for regression and $\cY = {1,\ldots, C}$ for $C$-class classification. In supervised learning, the process of learning a function mapping from $\cX$ to $\cY$ involves the learner specifying their inductive biases. These inductive biases include: (1) selecting a hypothesis class $\cH$, consisting of functions $f:\cX\to\cY$, (2) defining a suitable loss function $\ell: \cY\times\cY\to\RR_+$ based on the problem, (3) assuming the presence of a joint probability distribution $\PP$ over the variables $(X, Y)\in \cX\times \cY$ from which the data are sampled. Most critical to our work are (4) the assumptions regarding the deployment environment where the model $f$ is expected to generalise. 

\vspace{-.5em}
\subsection{Precise Learning}

In the following, we briefly review various generalisation assumptions commonly adopted in the literature and unify them under the setting of \emph{precise learning}.


\textbf{IID assumption. } Perhaps the most fundamental generalisation assumption in supervised learning is that the training and deployment environments are independent and identically distributed (IID). Under this assumption, a model that performs well in training is expected to generalize effectively in deployment. This concept is formalized by finding the function $f\in\cH$ that minimizes the population risk for $\PP$, known as the Bayes optimal model:
\begin{equation}
\cR(f) \triangleq \mathbb{E}_{\PP}[Z_f] = \mathbb{E}_{(X,Y)\sim \PP}\big[ \ell(f(X),Y) \big].
\end{equation}
For simplicity, we have denoted $Z\triangleq (X,Y)$, $\cZ\triangleq\cX\times\cY$, and $Z_f \triangleq \ell(f(X),Y)$ as the random loss associated with $f\in\cH$. In practice, since the true distribution $\PP$ is unknown, we focus on minimizing an empirical estimate of this risk based on IID samples $(x_i,y_i)_{i=1}^n$ from $\PP$, expressed as:
\begin{equation}
\widehat{\cR}(f) \triangleq \frac{1}{n}\sum_{i=1}^n\ell(f(x_i),y_i) + \eta\|f\|^2_{\cH}, \quad f\in\mathcal{H},
\end{equation}
where the second term is a regularization term to prevent overfitting, following the empirical risk minimization (ERM) principle \citep{Vapnik91:ERM,Vapnik95:SLT}.
This scenario introduces \textbf{data uncertainty}, stemming from the finite nature of data when approximating Bayes optimal models. The IID assumption also enjoys favourable guarantees, e.g., as the sample size $n$ increases, the uniform convergence of $\widehat{\cR}(\cdot)$ over $\mathcal{H}$ ensures that the gap between the empirical and the population risk becomes negligible with high probability; see, e.g., \citet{Vapnik98:SLT,Cucker02:SLT}.

\textbf{Beyond IID assumptions. } The IID assumption is often not viable in real-world scenarios due to various factors such as distribution shifts \citep{Candela09:DS,Beery18:WildLife,Beery20:iWildCam,Koh21:WILDS}, sub-population shifts \citep{Santurkar21:Breeds,Yang23:Subpop-Shift}, adversarial attacks \citep{Szegedy13:Properties,Goodfellow14:Explaining}, strategic manipulation \citep{Hardt16:Strategic,Perdomo20:Performative,vo_causal_2023}, and time shifts \citep{Gagnon-Audet22:WOODSBF}. In response to these challenges, learners must consider \textbf{generalisation uncertainty} when designing their learning algorithm. This uncertainty is typically represented by a credal set $\KK(Z)~$\citep{Walley91:SRIP}, a closed set of potential probability distributions that reflect the learner's ignorance, or partial knowledge about the deployment environments.

For example, in distributionally robust optimisation \citep{rahimian_distributionally_2022}, the credal set comprises distributions within an $\epsilon$ distance from the empirical distribution, and the goal is to optimise $f$ for the worst-case empirical risk within it. Another approach, involves learning across multiple domains $\PP_1,\dots,\PP_d$, and assumes the deployment distribution lies within their convex hull \citep{mansour_multiple_2012,krueger_out--distribution_2021,foll_gated_2023}. 
In invariant causal prediction \citep{Peters16:ICP,HeinzeDeml18:ICP-Nonlinear}, hypothetical interventional distributions associated with a structural causal model (SCM) constitute the credal set.
Learning algorithms here aim to optimize for worst-case empirical risk \citep{Arjovsky19:IRM,Ben-Tal09:RO,Sagawa20:DRO,krueger_out--distribution_2021}, average-case empirical risk \citep{Blanchard11:Generalize,Blanchard21:DG-MTL,Muandet13:DG,Zhang21:ARM}, and interpolations thereof \citep{Eastwood22:QRM}. The choice of risk corresponds to selecting a particular distribution within the credal set, such as the centroid of the convex hull referring to the average case. Notably, the credal set in the IID case reduces to a single distribution, $\KK(Z) = \{\PP\}$. 

\subsection{Previous Work}


\textbf{Limitation of precise learning. } A majority of previous work in both IID and OOD generalisation falls into the precise learning setting. A fundamental requirement is for the learner to commit to a specific notion of generalisation. This involves precisely selecting a particular distribution $\PP\in\KK(Z)$ during training and performing statistical learning to develop the model $f$. 
Although widely used, this might not always be optimal in modern machine learning settings, especially when there is a clear institutional separation between those who build and those who operate the model (cf. \Cref{sec:imprecise-dg}). This separation presents two significant challenges. First, it assumes that the learners either fully understand the specific generalisation needs of the operators, or that the operators have comprehensive access to the datasets and a thorough understanding of statistical inference, effectively making them the learners. Second, the choice of generalisation strategy is inherently subjective, involving normative decisions by the operators. For instance, a risk-averse operator might lean towards a worst-case empirical risk optimiser, while an operator with in-depth knowledge of the deployment environment might prefer an average-case empirical risk optimiser.


\textbf{Domain generalisation strategies. } The core of domain generalisation is the invariance principle~\citep{Muandet13:DG,Arjovsky19:OOD}, which asserts that certain properties remain constant across different environments and thus are expected to generalise to unseen settings. This principle is reflected in approaches focusing on feature representation \citep{Muandet13:DG,Ghifary15:DG-MTA,Arjovsky19:IRM}, causal mechanism \citep{Peters16:ICP,Rojas-Carulla18:CausalTransfer,Heinze-Deml21:CVP}, and risk functional \citep{krueger_out--distribution_2021}, all aimed at identifying and leveraging these invariant properties. While necessary, this principle faces two challenges: it abstracts away the inherent heterogeneity across environments~\citep{Heckman01:Heterogeneity}, which might give rise to non-invariant yet generalisable properties~\citep{Eastwood23:Spuriosity,Nastl24:Causal-Predictors}. Furthermore, identifying and utilising invariant properties faces practical difficulties due to their need for large sample sizes~\citep{Rosenfeld21:IRM-Risk,Kamath21:IRM-I}.


Addressing these challenges, recent work suggests combining domain-invariant with domain-specific properties~\citep{Liu21:HRM,Mahajan21:DG-CM}. While these approaches have been shown to improve in-domain generalisation performance, how domain-specific properties affect OOD generalisation in unseen environments remains unclear. To overcome this, it is popular to utilise various forms of test-time adaptation via auxiliary tasks~\citep{Sun20:Test-Time,Wang21:Tent,Chen23:Improved-TTT}. However, \citet{Liu21:TTT} has shown that this strategy can improve the pre-trained model only when the auxiliary loss aligns with the main loss. This suggests that a certain degree of precision in aligning losses is essential for effective domain generalisation.



\textbf{Generalisation uncertainty representation. } As opposed to the credal set $\KK(Z)$, some authors have instead adopted a second-order probability (aka meta distributions) as a belief over the ``true'' or ``ideal'' probabilities $\PP(Z)$ \citep{Blanchard11:Generalize,Blanchard21:DG-MTL,Muandet13:DG,Eastwood22:QRM}. However, \citet[Sec. 5.10]{Walley91:SRIP} pointed out that if probability distributions entail behavioural dispositions, it is necessary that the credal set must collapse to a singleton to avoid incoherent behaviour. 
Paradoxically, this implies that assuming the existence of meta-distributions is equivalent to making the IID assumption in the first place, emphasising that one must clearly differentiate generalisation uncertainty from data uncertainty.

\textbf{Learning under imprecision.} Machine learning inherently grapples with imprecision due to its inductive nature. One common approach to mitigate this is to create precision at various stages of model development. Techniques like data up/downsampling~\citep{he2009learning} and fusion~\citep{chau_deconditional_2021,chau_bayesimp_2021} address issues of granularity and missing data by drawing information from a precise empirical distribution. During algorithm selection, approaches like the Bayesian paradigm, ensembling, and AutoML~\citep{he2021automl} are used to handle potential model misspecification by selecting a precise model from a set of alternatives. Furthermore, model deployment requires a precise definition of generalisation, such as optimising for average-case or worst-case risks, to be determined before training.


When the introduced precision is not warranted, imprecise probabilists advocate for learning \emph{along with} imprecision. For instance, Walley's Imprecise Dirichlet Model effectively handles incomplete and missing data~\citep{utkin_imprecise_2021}. Dempster-Shafer Theory~\citep{shafer1992dempster} enables the fusion of multiple information sources, considering all available evidence. Credal learners, including credal decision trees~\citep{abellan2010ensemble}, credal networks~\citep{cozman2000credal}, and imprecise Bayesian neural networks~\citep{caprio_imprecise_2023}, propagate imprecision to prediction, resulting in models that capture the full range of possible outcomes. Central to these methods is the concept of a set of permissible solutions. This approach leads to indeterminate yet credible models, particularly in domains where uncertainty is prevalent. Our research aligns with this line of work, focusing on developing domain generalisation strategies that acknowledge and adapt to imprecision. By embracing imprecision, we aim to create models that offer a range of permissible solutions, empowering model operators to make informed choices at test time. The use of a credal set to model epistemic uncertainty has been concurrently explored by \citet{Caprio24:CLT} to derive generalization bounds under credal uncertainty.

\section{Imprecise Domain Generalisation}
\label{sec:imprecise-dg}

In this work, we advocate for an \emph{imprecise learning}, where learners do not commit to any particular $\PP\in\KK(Z)$ at training time, but express their uncertainty through a credal set $\KK(Z)$, where we discuss our choice in \Cref{sec:cvar}. We operationalise this idea in the context of domain generalisation (DG) problems. To this end, consider data coming from $d$ distinct domains, each with its own distribution $\PP_1,\dots,\PP_d$, and corresponding risk profiles $(\cR_1,\ldots,\cR_d) \triangleq \bm{\cR}$. The learner's objective is to select an optimal hypothesis from $\cH$ considering both the risk profiles and $\KK(Z)$, based on a certain optimality criterion defined below. While we mainly focus on multi-domain environments, this framework is also relevant and adaptable to single-domain scenarios (see \Cref{sec:single-domain} for further discussion).


\textbf{Credal set and partial preference. } A crucial distinction between precise and imprecise learning lies in their approach to learner's preferences~\citep{chau2022spectral_ecml,chau_learning_2022}. Precise learners commit to a specific distribution $\PP\in\KK(Z)$ during training, creating a complete\footnote{For every $f, g \in \cH$, either $f\succeq g$, $g\succeq f$, or both hold.} and transitive preference order $\succeq$ based on empirical risk in $\cH$. That is, for any $f, g\in \cH, f\succeq g$ if and only if $\widehat{\cR}(f) \leq \widehat{\cR}(g)$. Conversely, imprecise learning based on the credal set $\KK(Z)$ results in a \emph{partial} order over $\cH$ \citep{giron_quasi-bayesian_1980,Walley91:SRIP}:
\begin{lemma}\label{lem:incomplete-pref}
    The binary relation $\succeq$ represented by $\KK(Z)$ is such that for $f, g\in\cH$, $f\succeq g$, if and only if  $\mathbb{E}_{\PP}[Z_f] \leq \mathbb{E}_{\PP}[Z_g]$ for every $\PP\in\KK(Z)$.
\end{lemma}
This leads to an \emph{incomplete} preference ordering. \Cref{lem:incomplete-pref} highlights the challenge of learning with imprecision, implying that unless the learners are willing to exert their judgement over the distributions in $\mathbb{K}(Z)$, as was previously done in precise learning, it is no longer possible to unanimously identify the ``best'' hypothesis in $\cH$ from the observed data alone. 
In the following, we describe how the learners can implement imprecise learning at training time such that the operators can make prediction efficiently at test time.

\subsection{Aggregation Functions and Optimality Criteria}

To facilitate learning, we need a certain notion of optimality taking into account $\KK(Z)$. We formalise this by considering an \emph{aggregated learning algorithm} $\mfA: \bm{\cR}\mapsto h^*$ that takes a risk profile and returns a hypothesis $h^*\in\cH$. In particular, we focus on a specific type of aggregation function called an aggregated risk minimizers:
\begin{equation}\label{eq:arm}
    \mfA: \bm{\cR} \mapsto \arg\min_{\theta\in\Theta}\, \rho_{\lambda}[\bm{\cR}](h_{\theta}), \quad \lambda\in\Lambda,
\end{equation}
where $\rho_\lambda : \cL_2^d(\cH)\to \cL_2(\cH)$ is a risk aggregation function indexed by $\lambda\in\Lambda$, which yields the non-negative real-valued statistical functional $\rho_\lambda[\bm{\cR}]:\cH\to\RR_+$. Here, we assume that our model class $\mathcal{H}$ is parametrized by a parameter space $\Theta\subseteq \RR^p$, e.g., a weight vector in a neural network.
We call $\Lambda$ a choice space which arises exclusively due to the imprecision of the learning problem (cf. \Cref{lem:incomplete-pref}) and serves merely as an index set. 
In practice, we only have access to the empirical risks $(\widehat{\cR}_1,\ldots,\widehat{\cR}_d) \triangleq \widehat{\bm{\cR}}$ which we can substitute directly into \eqref{eq:arm}. 
In \cref{sec:cvar}, we consider Conditional Value-at-Risk (CVaR) as a concrete example of the risk aggregator $\rho_{\lambda}$.
Our formulation \eqref{eq:arm} is not only pertinent to financial risk measures but has also gained traction for creating interpretable, risk-aware machine learning algorithms \citep{pmlr-v97-williamson19a}.

For each $\lambda\in\Lambda$, we denote the Bayes optimal models by $h_{\lambda}^* \in \arg\min_{\theta\in\Theta}\, \rho_{\lambda}[\bm{\cR}](h_{\theta})$ and the associated parameter by $\theta^*_{\lambda}$.
Unfortunately, for continuous choice space, it is unrealistic for the learner to find the Bayes optimal models in $\cH$ simultaneously for all $\lambda\in\Lambda$. For this reason, we generalise the notion of Pareto optimality \citep{Pareto1897:Pareto} from multi-objective optimisation to its continuous counterpart and propose an alternative optimality criterion with respect to all $\lambda\in\Lambda$ : C-Pareto optimality. 

\begin{definition}[C-Pareto optimality]\label{def:c-pareto-optimality}
The hypothesis $h_{\theta}$ dominates $h_{\theta'}$, denoted by $h_{\theta} \triangleright h_{\theta'}$, if $\rho_{\lambda}[\bm{\cR}](h_{\theta}) \leq \rho_{\lambda}[\bm{\cR}](h_{\theta'})$ for all $\lambda\in\Lambda$ and $\rho_{\tilde{\lambda}}[\bm{\cR}](h_{\theta}) < \rho_{\tilde{\lambda}}[\bm{\cR}](h_{\theta'})$ for some $\tilde{\lambda}\in\Lambda$. Then, $h_{\theta}$ is C-Pareto optimal if there exists no $h_{\theta'}$ such that $h_{\theta'} \triangleright h_{\theta}$.
\end{definition}

When $\lambda$ takes values on a finite set $\Lambda$ with $m$ elements, i.e., $\Lambda=\{\lambda_1,\ldots,\lambda_m\}$, \Cref{def:c-pareto-optimality} coincides with the Pareto optimality in standard multi-objective optimization (MOO); see, e.g., \citet{Sener18:MTL-MOO,Lin19:Pareto-MTL,Zhang20:HyperMOO,Ma20:Efficient-Pareto-MTL} and references therein. \citet{Chen23:Pareto-IRM} have recently studies trade-offs between ERM and existing OOD objectives using MOO.

It is not hard to show that, like $\succeq$ introduced in \Cref{lem:incomplete-pref}, $\triangleright$ can be incomplete and that any Bayes optimal models $h_{\lambda}^*$ are also C-Pareto optimal. Intuitively, instead of obtaining the Bayes optimal model for all $\lambda\in\Lambda$, the learner can at best find the non-dominating models, i.e., the models upon which an improvement is only possible at a cost of deterioration of another non-dominating model.



Next, we introduce the notion of C-Pareto stationary used to check if a model is C-Pareto optimal.

\begin{definition}[C-Pareto stationary]
    Suppose $\rho_{\lambda}[\bm{\cR}](h_{\theta})$ is a smooth function of $h_{\theta}$ and define the local gradient at $h_\diamond$ as $v_\lambda := \nabla\rho_{\lambda}[\bm{\cR}](h_\diamond)$. The point $h_\diamond$ is called C-Pareto stationary if and only if there exists a probability density $q$ such that $\int v_\lambda \,d q(\lambda) = 0$.
    \label{def:C-pareto-stationarity}
\end{definition}

\subsection{Conditional Value-at-Risk (CVaR)} 
\label{sec:cvar}

In theory, all aggregation functions $\rho_\lambda[\bm{\cR}]$ can be expressed as a type of weighted average of $\bm{\cR}$, as detailed in Proposition~\ref{prop:cvar_aggregate}. A high level of generality could be achieved by formulating $\KK(Z)$ as the convex hull of $\PP_1,\dots,\PP_d$. This corresponds to treating the choice parameter $\lambda \in \RR^d$ as all possible averaging weight, thus defining $\rho_\lambda[\bm{\cR}] = \lambda^\top \bm{\cR}$. However, this approach has its serious drawbacks, since $\lambda$ might be difficult for the operators to interpret, potentially leading to \emph{irrational} decisions. For instance, operators may inappropriately assign more weight to domains that are easier to train, resulting in atypical ``risk-seeking'' behaviour.

To select an appropriate aggregation function (equivalent to formulating an appropriate credal set) that is both interpretable and aligned with typical behaviour such as risk aversion, we opt for $\rho_\lambda$ from the class of risk measures. This corresponds to formulating credal set as distributions that are mixtures of $\PP_1,\dots\PP_d$ with weights determined by the aggregation function. Notably, we choose the Conditional Value-at-Risk (CVaR):
\begin{definition}[Conditional Value-at-Risk~\citep{rockafellarConditionalValueatriskGeneral2002}]
\label{prop:cvar}
Let $\bm{\cR} = ({\cR}_{1}, \dots, {\cR}_{d})$ represent our risk profile, and $F_{\bm{\cR}}(r) = \frac{1}{d}\sum_{i=1}^d\mathbb{I}[\cR_i\leq r]$ as the cumulative distribution function (CDF) for $\bm{\cR}$. Define $r_\lambda = \min_r\{r \mid F_{\bm{\cR}}(r) \geq \lambda\}$ as the $\lambda$-level quantile. Then, the Conditional Value-at-Risk for $\bm{\cR}$ at level $\lambda$ is given by:
\begin{align}
\sum_{i=1}^d \left(\eta_\lambda \mathbb{I}[\cR_i = r_\lambda] + \frac{(1-\eta_\lambda)\mathbb{I}[\cR_i \geq r_\lambda]}{{\sum_{i=1}^d}\mathbb{I}[\cR_i \geq r_\lambda]}\right)\cR_i
\end{align}
where $\eta_\lambda = (F_{\bf{\cR}}(r_\lambda) - \lambda)(1-\lambda)^{-1}$, indicating the discontinuity level of the CDF at $\lambda$.
\end{definition}

CVaR effectively enables operators to express their level of risk aversion through $\lambda$, which in turn influences the selection of riskier domains for optimization. Additionally, this approach provides a means to transition smoothly between two prevalent notions of generalisation \citep{Robey22:PRL,Eastwood22:QRM,Li23:TERM}, namely optimising average risks ($\lambda = 0$) and worst-case risks ($\lambda = 1$). Furthermore, CVaR belongs to a class of coherent risk measures, which possess desirable properties \citep{Artzner99:CRM} and have been studied in the robust optimisation literature; see, e.g., \citet{Ben-Tal10:Soft-Robust}.

\subsection{Augmented Hypothesis}

To further institutionalize the separation of statistical decision-making, i.e., choosing appropriate notion of generalisation (performed by the operator) from statistical learning (performed by the learner), we propose to shift the problem view to an imprecise setting where the learner does not assume a priori which $\lambda\in\Lambda$ is relevant to the operator, but instead designs a model that allows the operator to choose their own $\lambda$ at deployment time.\footnote{While we focus primarily on the learning aspect and assume throughout that the operator knows how to specify $\lambda$, we acknowledge the challenge of eliciting operators’ preferences at test time; see \Cref{sec:test-time-elicit} for further discussion on test-time elicitation.}

To this end, we extend the hypothesis space to an augmented hypothesis space $\cH_{\Lambda}$ of functions of both $x$ and $\lambda$, and propose to learn an \emph{augmented} hypothesis $\bar{h}_\xi : \cX\times\Lambda\to\cY$ parametrized by a parameter $\xi\in\Xi\subseteq\RR^q$. In contrast with $h_\theta\in\cH$ which is fixed across $\lambda\in\Lambda$, an augmented hypothesis $\bar{h}_\xi\in\cH_{\Lambda}$ describes a range of possible hypothesis $\bar h_\xi(\cdot, \lambda)$ for each $\lambda\in\Lambda$, such that the user can choose the one that best fits their needs. By abuse of notation, we consider $\rho_{\lambda}[\bm{\cR}](\bar{h}_{\xi}(\cdot,\lambda))$ as a point-wise aggregated risk for the augmented hypothesis $\bar{h}_{\xi}\in\cH_{\Lambda}$. The subtle difference here is that we evaluate the objective at $\bar{h}_{\xi}(\cdot,\lambda)$ for the same $\lambda$ used in $\rho_{\lambda}$. While the idea of augmented hypothesis with loss-conditional learning has previously been considered \cite{Brault19:Infinite-Task,Dosovitskiy20:YOTO}, existing work still fall into the setting of precise learning, as we describe subsequently in \Cref{sec:algorithm}.

The function $h^* : (x, \lambda)\mapsto h^*_\lambda(x)$ that maps onto a Bayes optimal model for each $\lambda\in\Lambda$ is an example of such augmented hypothesis. However, we may again prefer to consider a more amenable optimality criterion that seeks optimality jointly across $\Lambda$. To this end, we extend \Cref{def:c-pareto-optimality} to an augmented hypothesis.

\begin{definition}[C-Pareto optimal augmented hypothesis]
    The augmented hypothesis $\bar h_\xi$ dominates $\bar h_{\xi'}$, denoted $\bar h_\xi\triangleright \bar h_{\xi'}$, if $\rho_\lambda[\bm{\cR}]\big(\bar h_\xi(\cdot, \lambda)\big) \leq \rho_\lambda[\bm{\cR}]\big(\bar h_{\xi'}(\cdot, \lambda)\big)$ for all $\lambda\in\Lambda$ and $\rho_{\tilde \lambda}[\bm{\cR}]\big(\bar h_\xi(\cdot, \tilde \lambda)\big) < \rho_{\tilde \lambda}[\bm{\cR}]\big(\bar h_{\xi'}(\cdot, \tilde \lambda)\big)$ for some $\tilde\lambda\in\Lambda$. Then $\bar h_\xi$ is C-Pareto optimal if there exists no $\bar h_{\xi'}$ such that $\bar h_{\xi'}\triangleright\bar h_\xi$.
\end{definition}

We can again verify that a function $h^* : (x, \lambda)\mapsto h^*_\lambda(x)$ that maps onto a Bayes optimal model for each $\lambda$ is in fact C-Pareto optimal. The following result shows that, under the assumption of existence of a Bayes optimal model, C-Pareto optimality is in fact equivalent to Bayes optimality.
\begin{proposition}\label{prop:uniquenessofh}
    Suppose there exists $h^*\in\cH_\Lambda$ such that $h^*(\cdot, \lambda)$ is Bayes optimal for all $\lambda\in\Lambda$. Then an augmented hypothesis $g^*\in\cH_\Lambda$ is C-Pareto optimal if and only if $g^*(\cdot, \lambda)$ is a Bayes optimal model for all $\lambda\in\Lambda$.
\end{proposition}

Proposition~\ref{prop:uniquenessofh} illustrates that all C-Pareto optimal augmented hypotheses can simultaneously learn all the Bayes optimal models. While this provides a strong guarantee, finding a C-Pareto optimal solution may still in practice be challenging and, when possible, one will prefer optimising against a scalar objective. 

Let $\Delta(\Lambda)$ be the space of probability density functions over $\Lambda$. In our imprecise learning setting, a learner can scalarise the objective by choosing a distribution $Q\in\Delta(\Lambda)$, and taking an expectation over all objectives. This substitutes the learning problem over all of $\Lambda$ with the scalarised objective
\begin{equation}
    J_Q(\bar h_\xi) = \EE_{\lambda\sim Q}\left[\rho_\lambda[\bm{\cR}](\bar h_\xi(\cdot,\lambda))\right],
    \label{eq:aggregate-with-p-lambda}
\end{equation}
where the choice of distribution $Q$ corresponds to a choice of scalarisation from the learner.  The following proposition shows that all choices of $Q$ lead to Bayes optimal models on their support.
\begin{proposition}\label{imprecise-aggregation-result}
    Let $Q\in\Delta(\Lambda)$. If $g^*\in \cH_\Lambda$ solves the scalarised optimisation problem, i.e., $g^* \in {\arg\min}_{g\in \cH_\Lambda}\,J_Q(g)$,
    then $g^*(\cdot, \lambda)$ is a Bayes optimal model for all $\lambda\in\Lambda$ such that $Q(\lambda) > 0$.
\end{proposition}
    
A similar result has previously been shown in \citet[Proposition 1]{Dosovitskiy20:YOTO} under the continuity and infinite model capacity assumptions.
This result implies in particular that for any choice of distribution $Q$ with full support, the scalarised objective can in theory yields a Bayes optimal model for every $\lambda\in\Lambda$. 

\section{Imprecise Risk Optimisation}
\label{sec:algorithm}

Unfortunately, \Cref{imprecise-aggregation-result} does not inform specific choices of $Q$ for the learner, leaving them in a state of ignorance. Under this scenario, the most popular narrative in the literature is to leave the choice of $Q$ to the operators or to adopt non-informative priors such as Jeffreys prior and uniform priors \citep{Brault19:Infinite-Task,Dosovitskiy20:YOTO}.
However, both approaches would defeat the purpose of this work as they render the learning problem precise again (see the discussions in \Cref{sec:preliminaries}). In particular, it has been argued that complete or partial ignorance cannot be fully represented by a single precise probability \citep[Sec. 5.5]{Walley91:SRIP}. 
For example, uniform distribution is not an appropriate way of representing ignorance because it coincides with a precise judgement of uniform belief.

\textbf{C-Pareto improvement.} To overcome this challenge, we adopt the concept of \emph{C-Pareto improvement} which allows us to develope a learning algorithm that respects not only the limitation of evidence and resource, but also the complete ignorance of the learner. Specifically, we focus on the gradient-based method:
\begin{equation}\label{eq:gradient-update}
    \xi_{t} \leftarrow \xi_{t-1} - \eta\cdot\nabla_\xi J_{Q_t}(\bar h_\xi), \quad Q_t\in\Delta(\Lambda).
\end{equation}
We say that the update \eqref{eq:gradient-update} makes a C-Pareto improvement if $\bar{h}_{\xi_t}$ dominates $\bar{h}_{\xi_{t-1}}$ according to the aggregation $\rho_{\lambda}[\bm{\cR}]$. The central concept involves the adaptive selection of $Q_t$ at each step, ensuring that the parameter update remains consistently non-dominant. This approach bears resemblance to the multiple-gradient descent algorithm (MGDA) utilised in multi-objective optimisation \citep{MGDA12}. The subsequent result demonstrates the specific selection of $Q_t$ that results in C-Pareto improvement.

\begin{theorem}\label{theorem:pareto-improvement}
For $\lambda\in\Lambda$, suppose $\xi\mapsto\rho_\lambda[\bm{\cR}](\bar{h}_\xi(\cdot,\lambda))$ is locally continuously differentiable in a neighbourhood of $\xi$. Define 
\begin{equation}
Q_t^* \in \underset{Q \in \Delta(\Lambda)}{\arg\min}\, \left\|\nabla_{\xi_{t-1}} J_Q(\xi_{t-1}) \right\|_2 \label{eq:choose-qt}
\end{equation}
and $v_t(\xi_t) = \nabla_{\xi_t} J_{Q^*_t}(\xi_t)$. Then the update $\xi_{t} \leftarrow \xi_{t-1} - \eta\cdot v_{t}(\xi_t)$ for an appropriate choice of $\eta >0$ always makes C-Pareto improvement. 
\end{theorem}

\subsection{Practical Algorithm with Theoretical Justification}

In practice, we have access to data from $d$ distinct domains. The empirical risk for the augmented hypothesis $\bar h_\xi\in\cH_\Lambda$ for the $i$\textsuperscript{th} domain can be computed for each $\lambda\in\Lambda$ as
\begin{equation}
\widehat\cR_i(\bar h_\xi(\cdot, \lambda)) = \frac{1}{n}\sum_{j=1}^n \ell(\bar h_\xi(x_j^{(i)},\lambda),y_j^{(i)}),
    \label{eq:augmented-risk-functional}
\end{equation}
where $(x_j^{(i)},y_j^{(i)})\sim \mathbb{P}_i$. In principle, the choice of $\lambda$ determines how to aggregate the risk profile. However, in practice once $\lambda$ is known, only then the corresponding $\bar{h}_\xi(\cdot,\lambda)\in\cH_{\Lambda}$ can be used to compute the empirical risk profile. For a particular objective $\rho_{\lambda}[\bm{\widehat{\cR}}](\bar{h}_\xi(\cdot,\lambda))$, we can compute the corresponding empirical risk profile as $\bm{\widehat{\cR}}(\bar{h}_\xi(\cdot,\lambda))=\{\widehat{\cR}_i(\bar{h}_\xi(\cdot,\lambda)),\dots,\widehat{\cR}_d(\bar{h}_\xi(\cdot,\lambda))\}$. 
If $Q$ is known to the learner, they can sample $\{\lambda_j\}_{j=1}^m\sim Q$ and compute the corresponding empirical risk profiles for each $\lambda_j$. 
However, for an imprecise learner, the right choice of distribution $Q$ is unknown a priori. Therefore, we defer the computation of the empirical risk profile until the corresponding $\lambda$ is known. That is, given a candidate distribution  $Q \in \Delta(\Lambda)$ we compute the risk profile and aggregate it with $\{\lambda_j\}_{j=1}^m\sim Q$. We can then estimate $Q_t^*$ with $\widehat{Q}_t$ using Monte Carlo estimate of \eqref{eq:choose-qt}, i.e., 
\begin{equation}
    \widehat{Q}_t = \underset{Q \in \Delta(\Lambda)}{\arg\min}\left\|\frac{1}{m}\sum_{j=1}^m\nabla\rho_{\lambda_j}[\bm{\widehat{\cR}}](\bar{h}_\xi(\cdot,\lambda_j))\right\|_2,
    \label{eq:discrete-p-stationary-dist}
\end{equation}
where $\{\lambda_j\}_{j=1}^m\sim Q$. The direction of C-Pareto improvement is obtained by $\hat{v}_t(\xi_t) = \nabla_{\xi_t} J_{\widehat{Q}_t}(\xi_t)$. 
\Cref{iro-algorithm} summarises the proposed algorithm.

\begin{algorithm}[t]
\caption{Imprecise Risk Optimisation (IRO)}
\label{iro-algorithm}
\begin{algorithmic}[1]
    \STATE \textbf{Input:} Data from $d$ distinct domains $\{x_i^{(d)},y_i^{(d)}\}_{i=1}^n \sim \mathbb{P}_d(X,Y)$, a loss function $\ell:\cY\times\cY\to\RR_+$, a probability space $\Delta(\Lambda)$, a (augmented) hypothesis class $\cH_{\Lambda}$, risk aggregator $\rho_\lambda:\cL^d(\cH)\rightarrow\cL(\cH)$, number of Monte Carlo samples $m$.
    \STATE Initialise the parameter $\xi \in \Xi$.
    \REPEAT
    \STATE Estimate $Q^*_t$ with $\widehat{Q}_t$ by solving \eqref{eq:discrete-p-stationary-dist} by computing $\widehat{Q}_t = \underset{Q\in \Delta(\Lambda)}{\arg\min}\,\|\frac{1}{m}\sum_{j=1}^m\nabla\rho_{\lambda_j}[\bm{\widehat{\cR}}](\bar{h}_\xi(\cdot,\lambda_j))\|_2$
    where $\lambda_1,\ldots,\lambda_m \sim Q$.
    \STATE Compute $\hat{v}_t(\xi) = \frac{1}{m'}\sum_{k=1}^{m'}\nabla \rho_{\lambda_k}[\bm{\widehat{\cR}}](\bar{h}_\xi(\cdot,\lambda_k))$ where $\lambda_1,\ldots,\lambda_{m'} \sim \widehat{Q}_t$.
    \STATE Update $\xi = \xi - \eta \hat{v}_t(\xi)$.
    \UNTIL{$\|\hat{v}_t(\xi)\|_2>\epsilon$}
\end{algorithmic}
\end{algorithm}

\begin{proposition}\label{prop:regret-bound}
    Let $Q\in\Delta(\Lambda)$ and let $\lambda_{\text{\textnormal{op}}}\in\Lambda$ such that $Q(\lambda_{\text{\textnormal{op}}}) > 0$. Assume that $\rho_\lambda$ is a linear, idempotent aggregation operator and that the loss $\ell$ is upper bounded by $M \geq 0$. Let $n \geq 1$ be the number of samples we observe from each environment, assumed equal across environments. Then, there exists $q \in (0, 1)$ such that if
    \begin{equation}
        \hat g \in\underset{\bar{g}\in\cH_\Lambda}{\arg\min}\, \frac{1}{m}\sum_{i=1}^m \rho_{\lambda_i}[\bm{\widehat\cR}](\bar{g}(\cdot, \lambda_i))
    \end{equation}
    where $\lambda_1, \ldots, \lambda_m\sim Q$, then for any $\delta > q^m$, the following inequality holds with probability $1 - \delta$:
    \begin{align}
        \begin{split}
        & \big|\rho_{\lambda_{\text{\textnormal{op}}}}[\bm{\cR}](\hat g(\cdot, \lambda_{\text{\textnormal{op}}})) - \rho_{\lambda_{\text{\textnormal{op}}}}[\bm{\cR}](h^*(\cdot, \lambda_{\text{\textnormal{op}}}))\big | \\
        & \qquad \leq 2M \left(\sqrt{\frac{\log(6/\eta_\delta)}{2n}} + \sqrt{\frac{\log(6/\eta_\delta)}{2 m (1-q)(1-q^m)}}\right),
        \end{split}
    \end{align}
    where $\eta_\delta = (\delta - q^m)/(1-q^m)$.
\end{proposition}

This proposition shows that even when the learner does not know the operator's true preference $\lambda_\text{op}$, the operator excess risk on the solution of the empirical scalarised IRO problem $\hat g$ is bounded with high probability in $O(n^{-1/2} + m^{-1/2})$, provided $Q$ has full support. This means in particular that, provided an unlimited budget on the number of samples (the $\lambda_i$s) that can be drawn from $Q$, the operator excess risk has a bound that matches standard learning rates for ERM.

The constant $q\in (0, 1)$ depends on the choice of distribution $Q$ and the operator's true preference $\lambda_\text{op}$. If $Q$ has a high density around $\lambda_\text{op}$, then $q$ can be chosen closer to zero. Conversely, if $Q$ has a lower density around $\lambda_\text{op}$, the values of $q$ will be closer to one, requiring a larger number of samples $\lambda_1, \ldots, \lambda_m$ to achieve a comparable bound.


\begin{figure*}
    \centering
    \hfill 
    \begin{subfigure}[b]{0.32\textwidth}
        \includegraphics[width=\textwidth]{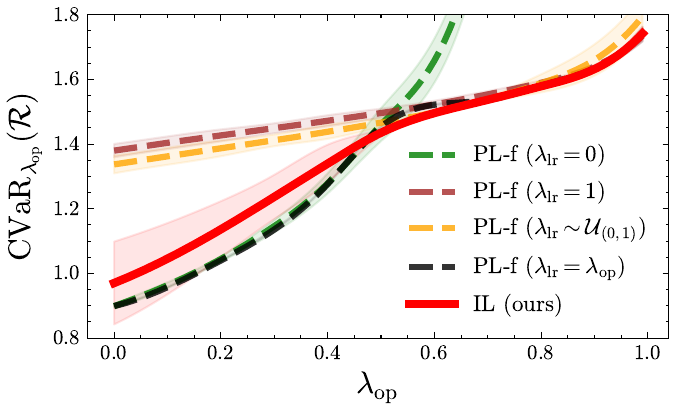}
        
        \caption{(Synthetic data) Comparing \textbf{IL} with $\textbf{PL-}f$ trained using different $\lambda_{\text{lr}}$.}
        \label{fig:precise-vs-imprecise}
    \end{subfigure}
    \hfill
    \begin{subfigure}[b]{0.32\textwidth}
        \includegraphics[width=\textwidth]{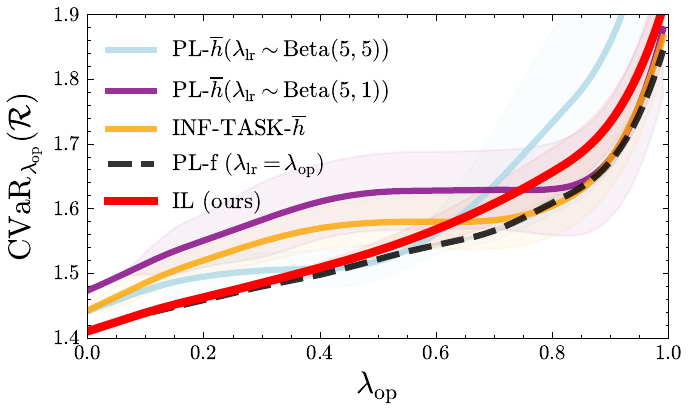}
        \caption{(Synthetic data) Comparing \textbf{IL} with $\textbf{PL-}\bar{h}$ trained using different priors over  $\lambda_{\text{lr}}$.}
        \label{fig:beta-ablation-main-paper}
    \end{subfigure}
    \hfill
    \begin{subfigure}[b]{0.32\textwidth}
        \includegraphics[width=\textwidth]{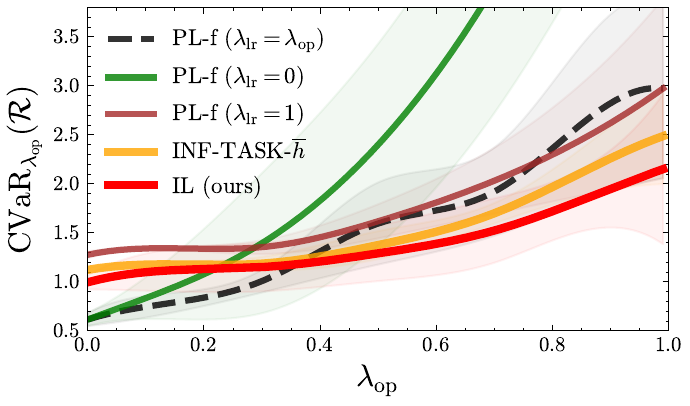}
        \caption{(UCI Bike Rentals) Comparing \textbf{IL} with various $\textbf{PL-}\bar{h}$ and $\textbf{PL-}f$.}
        \label{fig:bike-rentals}
    \end{subfigure}
    \vspace{-.5em}
    \caption{Experiments comparing imprecise learning (\textbf{IL}) with various precise learners with precise hypothesis ($\textbf{PL-}f$) and with augmented hypothesis ($\textbf{PL-}\bar{h}$). 1 standard deviation is included and experiments are repeated 5 times.}
    \label{fig:evaluation_big}
\end{figure*}

\section{Experiments}
\label{sec:experiments}


Our framework features a learner, who trains the model, and an operator, who employs it, with their preferences denoted as $\lambda_{\text{lr}}$ and $\lambda_{\text{op}}$. Due to the institutional separation, the operator's preferred generalisation strategy cannot be communicated to the learner. We assess our Imprecise Learning (IL) framework, allowing learners to train an augmented hypothesis $\bar{h}$ using our IRO algorithm (see \Cref{iro-algorithm}), enabling operators to provide $\lambda_{\text{op}}$ at deployment. This contrasts with Precise Learners (PL-$f$) who commit to a fixed generalisation ($\lambda_{\text{lr}}$) during training, producing a precise hypothesis $f:\cX\to\cY$, and PL-$\bar{h}$, who create an augmented hypothesis but with a pre-determined prior over $\lambda_{\text{lr}}$.

We evaluate using the objective $\rho_{\lambda_{\text{op}}}[\bm{\cR}]$, comparing \textbf{IL}, $\textbf{PL-}f$ with fixed (0 or 1) or uniform $\lambda_{\text{lr}}$, and $\textbf{PL-}\bar{h}$ with prior of $\lambda_{\text{lr}}$ as Beta distributions (5,5), (5,1), and (1,1). The strategy aligning with $\text{Beta}(1, 1)$ corresponds to the approach in \citet{Brault19:Infinite-Task}, thus is termed \textbf{INF-TASK-}$\bar{h}$. We benchmark against an ideal scenario where $\lambda_{\text{lr}}$ equals $\lambda_{\text{op}}$ and also calculate the maximum regret, i.e., for any model $\bar{h}$ (or $f$), $\text{max-regret}(\bar{h}) \triangleq \sup_{\lambda_{\text{op}}\in \Lambda}\rho_{\lambda_{\text{op}}}[\bm{\cR}](\bar{h})-\rho_{\lambda_{\text{op}}}[\bm{\cR}](\bar{h}^*_{\lambda_{\text{op}}})$, to gauge the models' deviation from optimality across all $\lambda_{\text{op}}$.

\textbf{Synthetic data}: Following \citet{eastwood2022}, we construct a simulated experiment to compare learners. We consider a linear model for each domain $d$: $Y_{d}=\theta_{d}X+\epsilon$ with $X \sim \mathcal{N}(1,0.5)$ and $\epsilon\sim \mathcal{N}(0,0.1)$. We simulate different domains by drawing $\theta_d$ with probability $p=0.5$ from Uniform distributions $\mathcal{U}_{(1,1.1)}$ and $\mathcal{U}_{(-1.1,-1)}$. This allows data to exhibit multi-modality, thus creating a discontinuous risk profile which becomes harder for a single augmented hypothesis to capture. We consider 250 train and 250 test domains with 100 samples from each domain.

\textbf{CMNIST dataset.} We also experiment on the CMNIST dataset~\cite {arjovski2021}, which is a modified version of the MNIST dataset. The task is to classify digits $\{0,1,2,3,4\}$ and $\{5,6,7,8,9\}$ into negative and positive classes, respectively. A color is introduced as an additional domain-specific predictive feature that varies across domains, e.g., $\PP(Y=1\,|\,\text{color}=\text{red}) = 0.9$ for domains in which the true label is highly correlated with the color feature. As a result, the mechanism by which color influences the label changes across domains, but the shape has a stable mechanism across domains (see \Cref{fig:dag_of_cmnist}). We sample 10 training environments from a long-tailed $\text{Beta}(0.9,1)$ distribution, resulting in over-represented (majority) and under-represented (minority) subgroups (see \Cref{fig:dist_of_cmnist_train}). Note that we do not make the IID assumption over environments since we evaluate all subgroups at test time. We further discuss the dataset and experiment setup in \Cref{appendix:cmnist}.

\textbf{Real-world data:} Following \citet{rothenhausler2021anchor} and \citet{subbaswamy2019preventing}, we use the UCI Bike Sharing dataset~\citep{fanaee2014event} to predict the number of hourly bike rentals $R$ from various weather-related features. Here, $R$ is transformed from count to continuous with normalization. The data contains $17,379$ observations with temporal information such as season and year. The data is partitioned by season (1-4) and year (1-2) to create $8$ different domains. Domains from the first year are used for training and the subsequent year as test domains.

\begin{table}[t!]
    \caption{Reporting the maximum regret averaged over 5 repetitions for each experiment with one standard error included. \textbf{Top}: Comparing \textbf{IL} with $\textbf{PL-}f$ (Synthetic). \textbf{Middle}: Comparing \textbf{IL} with $\textbf{PL-}\bar{h}$ (Synthetic). \textbf{Bottom}: Comparing \textbf{IL} with $\textbf{PL-}f$ and $\textbf{PL-}\bar{h}$ (Bike Rentals).}
    \vspace{-1em}
    \centering
    \resizebox{\columnwidth}{!}{
    \begin{tabular}{llll}
        \toprule 
        $\textbf{PL-}f~(\mathcal{U}(0,1))$ & $\textbf{PL-}f~(\lambda_{lr}=0)$ & $\textbf{PL-}f~(\lambda_{lr}=1)$ & \textbf{IL} (ours) \\
        \toprule 
        1.971 $\pm$ {\tiny (0.0098)} & 6.177 $\pm$ {\tiny (0.0617)} & 2.010 $\pm$ {\tiny (0.0564)} & \textbf{0.867} $\pm$ {\tiny (0.0058)} \\
        \midrule
        \toprule $\textbf{PL-}\bar{h}~($Beta(5,5)) & $\textbf{PL-}\bar{h}~($Beta(5,1)) & INF-TASK-$\bar{h}$ & \textbf{IL} (ours)\\
        \hline 
         1.79 $\pm$ {\tiny (0.12)} & 1.57 $\pm$ {\tiny (0.03)} & 0.935 $\pm$ {\tiny (0.04)} & \textbf{0.56} $\pm$ {\tiny (0.00)} \\
        \midrule
        \toprule 
        $\textbf{PL-}f(\lambda_{lr}=0)$ & $\textbf{PL-}f(\lambda_{lr}=1)$ & INF-TASK-$\bar{h}$ & \textbf{IL} (ours) \\
        \hline 
        4.81$\pm$ {\tiny{0.27}} & 0.66 $\pm$ {\tiny{0.01}}& 0.72 $\pm$ {\tiny{0.13}} & \textbf{0.42} $\pm$ {\tiny{0.08}}\\
        \hline 
    \end{tabular}}
    \label{tab:regret}
\end{table}

\subsection{Insights from Experiments}
\textbf{Comparing \textbf{IL} with $\textbf{PL-}f$.} Our initial experiment on synthetic data contrasts Imprecise Learning (\textbf{IL}) with Precise Learners ($\textbf{PL-}f$) across different $\lambda_{\text{lr}}$ settings, including average-case ($\lambda_{\text{lr}} = 0$) and worst-case ($\lambda_{\text{lr}}=1$) scenarios. 
Results, shown in \cref{fig:precise-vs-imprecise}, indicate that $\textbf{PL-}f$ models achieve the lowest aggregated risk compared to other learners when $\lambda_{\text{lr}}=\lambda_{\text{op}}$. However, when $\lambda_{\text{lr}}\neq\lambda_{\text{op}}$, $\textbf{PL-}f$ then deviates from this ideal scenario, which is expected since $\textbf{PL-}f$ models are finely tuned to their specific $\lambda_{\text{lr}}$. 
Conversely, \textbf{IL} achieves aggregated risks that remain close to the ideal scenario across the spectrum of $\lambda_{\text{op}}$, matching or exceeding the worst-case PL-$f$ in risk-averse settings ($\lambda_{\text{op}} > 0.6$) and surpassing both average-case and worst-case $\textbf{PL-}f$ as $\lambda_{\text{op}}$ increases. Notably, \textbf{IL} achieved the lowest maximum regret (see \cref{tab:regret}), underscoring the advantage of imprecise learning in handling generalization uncertainty.


\textbf{Comparing IL with PL-$\bar{h}$.} In our second experiment using synthetic data, we evaluate \textbf{IL}'s augmented hypothesis trained using imprecise risk optimisation (IRO) against precise learners (PL-$\bar{h}$) employing various optimization strategies influenced by their subjective beliefs about $\lambda_{\text{op}}$. Results in \cref{fig:beta-ablation-main-paper} indicate \textbf{IL}'s performance is close to the ideal baseline across most $\lambda_{\text{op}}$ values, except at higher risk levels where INF-TASK and $\textbf{PL-}\bar{h}$ trained under Beta(5,1) excel. This outcome aligns with expectations, as INF-TASK uniformly aggregates objectives, favoring higher-risk scenarios, similar to Beta(5,1)'s weighting towards higher $\lambda$. Despite this, \textbf{IL} outperforms these methods across other $\lambda$ and achieves the lowest maximum regret (see \cref{tab:regret}), demonstrating the efficacy of the proposed method.

\textbf{Comparing DG algorithms on CMNIST.} In \Cref{table:cmnist}, we compare \textbf{IL} to other DG methods on three representative domains from minority and majority subgroups (see \Cref{fig:exp-setup}). The domains $e\in\{0.0,1.0\}$ demonstrate opposite mechanisms, i.e., in domain $e=0.0$, the color red is fully predictive of the negative class, whereas for $e=1.0$, it is fully predictive of the positive class. In domain $e=0.5$, color is uncorrelated with the target. We can see that \textbf{IL} can learn relevant features in context with appropriate $\lambda$ and generalises in all scenarios. By setting $\lambda=0$, the model operator can be less risk averse and generalise better to domains from the majority subgroup, as noted in the performance of \textbf{IL} for $e=0.0$. With $\lambda\rightarrow 1$, the model operator can be risk averse and generalise better to the minority subgroup and is also reflected in the performance of \textbf{IL} for $e\in\{0.5,1.0\}$. Furthermore, with $\lambda\rightarrow 1$, it performs similarly to the invariant learners. We discuss the results on all test domains in \Cref{table:cmnist_regret} in \Cref{appendix:cmnist}.

\textbf{Real-world experiment.} \Cref{fig:bike-rentals} demonstrates similar comparisons between \textbf{IL} and various $\textbf{PL-}f$ and $\textbf{PL-}\bar{h}$ models as in previous experiments. Notably, \textbf{IL} surpassed the ideal scenario at higher risk levels. This can be attributed to the fact that CVAR as an objective discards data from lower-risk environments (see \cref{sec:cvar}), thus the optimisation has lower statistical efficiency as risk level increases. Augmented hypothesis mitigates this downside because it is smooth in the $\lambda$ parameter by design, thus can ``borrow'' information from nearby risk regions. At last, \textbf{IL} consistently achieved the lowest maximum regret as shown in \cref{tab:regret}.

\begin{table}[t!]
\centering
\caption{Accuracy and maximal regret of different domain generalisation algorithms on the CMNIST test environments from $\PP(Y=1\,|\,\text{color}=\text{red})=e$ with $e\in\{0.0, 0.5,1.0\}$, respectively. The hypothetical best invariant and Bayes classifier are listed in \textbf{bold}. Domain-wise best acc \& regret are highlighted in \textcolor{applegreen}{\textbf{green}}. Bayes classifier is defined w.r.t. the IID learner trained for a particular environment}
\resizebox{\columnwidth}{!}{
\begin{tabular}{llllll}
\toprule
{\bfseries Objective} & {\bfseries Algorithm} & $e=0.0$ & $e=0.5$ &$e=1.0$ & {\bfseries Regret} \\
\midrule
Average & ERM  & \textcolor{applegreen}{\textbf{96.1}} & 59.2 & 28.3   & 72.7\\
\cline{2-5} 
\multirow{2}{*}{Worst} & GrpDRO & 54.1 & 64.5  &\textcolor{applegreen}{\textbf{75.5}} & 46.9\\
& SD & 52.1 & 63.7 &73.3 & 47.9\\
\cline{2-5} 
\multirow{5}{*}{Invariance} 
&IGA & 71.8 & 65.2 &50.3 & 49.7\\
& IRM  & 72.0 & 69.7 & 67.7 & 32.3\\
&VREx & 72.7 & 69.5 &68.5& 31.5 \\
&EQRM & 67.8 & 69.1 &72.1 & 32.2\\
\cline{3-5}
& Oracle & \multicolumn{3}{c}{\textbf{73.5}} & \textbf{27.9}\\
\cline{2-5} 
\textbf{PL}-$\bar{h}$& Inf-Task & 96.0 & 63.1& 68.3 & 31.7 \\
\cline{2-5} 
\textbf{IL} (Ours) & IRO & 95.8& 69.5 & 70.3 & \textcolor{applegreen}{\textbf{29.7}} \\
\cline{2-5} 
Bayes & ERM (IID) & \textbf{100.0}& \textbf{75.0} & \textbf{100.0} \\
\bottomrule
\end{tabular}}
\label{table:cmnist}
\end{table}

\section{Conclusion}
\label{sec:conclusion}

In out-of-distribution (OOD) generalisation, a clear institutional separation between machine learners and model operators creates generalisation uncertainty that prevents consensus on a specific generalisation approach during training. 
To overcome this, we presented imprecise domain generalisation. 
Our approach incorporates imprecise risk optimisation, allowing learners to maintain imprecision during training, coupled with a model framework that lets operators specify their generalisation strategy at deployment. Both theoretical analysis and experimental evaluations demonstrate the effectiveness of our proposed framework. 

Our work faces two main limitations. First, it assumes that model operators are aware of their level of risk aversion. In practice, they may however struggle to precisely articulate their preferences. Consequently, this necessitates preference elicitation at test time, which may result in a probability distribution over $\lambda$ rather than a single value. Second, imprecise learning is more computationally intensive compared to precise counterparts as it involves optimising for a continuum of objectives. In our future work, we aim to broaden the scope of imprecise learning by implementing methods to elicit user preferences more effectively, improving computational efficiency, and exploring alternative aggregation functions. This approach would empower operators to weigh various criteria such as fairness, privacy, and algorithmic performance effectively.

\section*{Impact Statement}
This paper presents work whose goal is to advance the field of 
Machine Learning. There are many potential societal consequences 
of our work, none which we feel must be specifically highlighted here.
\section*{Acknowledgement}
We thank Kevin Murphy, Chris Holmes, Uri Shalit, Emtiyaz Khan, Hugo Monz\'on, Shai Ben-David, and Amartya Sanyal for fruitful discussion, and anonymous reviewers for their insightful feedback. We are indebted to Simon F\"oll for his contribution in conducting an initial set of experiments. 


\bibliographystyle{abbrvnat}
\bibliography{refs,ref_alan,ref_krik,ref_krik_ext}

\newpage
\appendix
\onecolumn

\appendix

\section{Proofs}

This section provides the detailed proofs of our main results presented in the paper.

\subsection{Proof of Proposition ~\ref{prop:uniquenessofh}}
\begin{proof}
\strut
\newline
($\Leftarrow$)

We can easily verify that if $g^*(\cdot, \lambda)$ is Bayes optimal for all $\lambda\in\Lambda$, then it is C-Pareto optimal.

\strut
\newline
($\Rightarrow$)

    Suppose $g^*\in\cH_\Lambda$ is C-Pareto optimal. We have
    \begin{align*}
        g^* \text{ C-Pareto optimal} & \Leftrightarrow \not\exists h\in\cH_\Lambda, h \triangleright g^*  \\
        & \Leftrightarrow \forall h\in\cH_\Lambda, \neg (h\triangleright g^*) \\
        & \Leftrightarrow \forall h\in\cH_\Lambda, 
            \big[\exists \lambda\in\Lambda, \rho_\lambda[\bm{\cR}](h(\cdot, \lambda)) > \rho_\lambda[\bm{\cR}](g^*(\cdot, \lambda))\big] \\
            & \qquad\qquad\enspace\, \lor \big[\forall\tilde\lambda\in\Lambda, \rho_{\tilde\lambda}[\bm{\cR}](h(\cdot, \tilde\lambda)) \geq \rho_{\tilde\lambda}[\bm{\cR}](g^*(\cdot, \tilde\lambda))\big].
    \end{align*}

    This implies in particular that
    \begin{equation*}
        \big[\exists \lambda\in\Lambda, \rho_\lambda[\bm{\cR}](h^*(\cdot, \lambda)) > \rho_\lambda[\bm{\cR}](g^*(\cdot, \lambda))\big] \lor \big[\forall\tilde\lambda\in\Lambda, \rho_{\tilde\lambda}[\bm{\cR}](h^*(\cdot, \tilde\lambda)) \geq \rho_{\tilde\lambda}[\bm{\cR}](g^*(\cdot, \tilde\lambda))\big].
    \end{equation*}
    Since $h^*(\cdot, \lambda)$ is Bayes optimal for all $\lambda\in\Lambda$, the first statement cannot be true. Therefore, the second statement must hold and we have
    \begin{align*}
    g^* \text{ C-Pareto optimal} & \Rightarrow \forall \tilde\lambda\in\Lambda, \rho_{\tilde\lambda}[\bm{\cR}](h^*(\cdot, \tilde\lambda)) \geq \rho_{\tilde\lambda}[\bm{\cR}](g^*(\cdot, \tilde\lambda)) \\
    & \Rightarrow \forall \tilde\lambda\in\Lambda, \rho_{\tilde\lambda}[\bm{\cR}](h^*(\cdot, \tilde\lambda)) = \rho_{\tilde\lambda}[\cR](g^*(\cdot, \tilde\lambda)) & (h^*(\cdot, \lambda) \text{ Bayes optimal}) \\
    & \Rightarrow \forall \tilde\lambda\in\Lambda, g^*(\cdot, \tilde\lambda) \text{ Bayes optimal } .
    \end{align*}
    This concludes the proof.
\end{proof}
\subsection{Proof of Proposition~\ref{imprecise-aggregation-result}}
\begin{proof}
    Since $h^*(\cdot, \lambda)$ is a Bayes optimal model for all $\lambda\in\Lambda$,  we have
    \begin{align*}
         \rho_\lambda[\bm{\cR}](g^*(\cdot, \lambda)) - \rho_{\lambda}[\bm{\cR}](h^*(\cdot, \lambda)) & \geq 0, \enspace\forall\lambda\in\Lambda \\
        \Rightarrow  \EE_{\lambda\sim Q}[\rho_\lambda[\bm{\cR}](g^*(\cdot, \lambda)) - \rho_{\lambda}[\bm{\cR}](h^*(\cdot, \lambda))] & \geq 0.
    \end{align*}

    But by definition of $g^*$ we also have
    \begin{align*}
        J_Q(g^*) & \leq J_Q(h^*) \Rightarrow \EE_{\lambda\sim Q}[\rho_\lambda[\bm{\cR}](g^*(\cdot, \lambda)) - \rho_{\lambda}[\bm{\cR}](h^*(\cdot, \lambda))] \leq 0.
    \end{align*}

    Therefore,
    \begin{equation*}
       \int_\Lambda \big[\rho_\lambda[\bm{\cR}](g^*(\cdot, \lambda)) - \rho_{\lambda}[\bm{\cR}](h^*(\cdot, \lambda))\big] Q(\lambda) \d\lambda = 0.
    \end{equation*}
    Since the integrand is positive, it implies that for all $\lambda\in\Lambda$ such that $Q(\lambda) > 0$, $\rho_\lambda[\bm{\cR}](g^*(\cdot, \lambda)) = \rho_{\lambda}[\bm{\cR}](h^*(\cdot, \lambda))$ which concludes the proof.
\end{proof}

\subsection{Proof of \Cref{theorem:pareto-improvement} and C-Pareto Improvement}
\label{appendix:c-pareto}

When the choice of scalarisation, i.e., $Q$ improves some objectives at the cost of degrading other objectives, it induces a preference. Therefore, the problem becomes multi-objective again as there will be a trade-off among these objectives. The imprecise choice of scalarization will be the distribution $Q^*$ such that it improves at least one of the objectives without degrading any other objective, i.e., it ensures C-Pareto improvement. Formally, 
\begin{proposition}
   Suppose a learning algorithm $\mfA$ learns an augmented hypothesis $g\in \cH_{\Lambda}$ using aggregated objective $J_Q(g)$ (ref. Eq~\ref{eq:aggregate-with-p-lambda}). Then, the learner does not induce an additional preference over $\cH_{\Lambda}$ if it makes Pareto improvement.
\label{prop:noaddtionalpreference}
\end{proposition}
\begin{proof}
    Consider a learner $\mfA$ which learns an augmented hypothesis $g\in\cH_{\Lambda}$ which aggregates the objectives $\rho_\lambda[\bm{\cR}](g(\cdot,\lambda))$ for all $\lambda\in\Lambda$ with respect to $Q$ to obtain the aggregated objective $\mathbb{E}_{\lambda\sim Q}[\rho_\lambda[\bm{\cR}](g(\cdot,\lambda))]$. Assume that $\cG=\{g_i\}_{i=0}^n$ denotes the sequence of models that the learner obtains at every update while minimizing the aggregated objective. We know that $\forall i,j$ such that $i<j$ and $g_i, g_j\in \cG$
    \begin{align*}
        \mathbb{E}_{\lambda\sim Q}[\rho_\lambda[\bm{\cR}](g_i(\cdot,\lambda))]>\mathbb{E}_{\lambda\sim Q}[\rho_\lambda[\bm{\cR}](g_j(\cdot,\lambda))]
    \end{align*}
which defines a preference on $\cG\subset\cH_{\Lambda}$ with aggregated objective as the utility function $u_{Q}(g):=-\mathbb{E}_{\lambda\sim Q}[\rho_\lambda[\bm{\cR}](g(\cdot,\lambda))]$. This additional preference relation agrees with the original binary preference relation ($\succeq$) on $\cH_{\Lambda}$ which defines dominance and C-Pareto optimality if there does not exist $g_i, g_j\in \cG$ such that $i<j$, $g_i\nsucceq g_j$ and $g_j\nsucceq g_i$ with $u_{Q}(g_i)>{Q}(g_j)$. This implies that aggregated objective $u_{Q}$ must be such that for all $g_i, g_j\in \cG$ and $i<j$, $g_j\succeq g_i$. That is, $u_{Q}$ should make C-Pareto improvement to not induce any additional preference. 
\end{proof}

Therefore, we propose an alternate characterization of C-Pareto optimality based on the concept of C-Pareto improvement with the idea of local gradients. 
\begin{proposition}
    An augmented hypothesis $\bar{h}_\xi$ is C-Pareto optimal if and only if there exists no $w\in \Xi$ such that for an $\epsilon>0$, $\rho_{\lambda}[\bm{\cR}](\bar{h}_{\xi-\epsilon w}(\cdot, \lambda)) \leq \rho_{\lambda}[\bm{\cR}](\bar{h}_{\xi}(\cdot, \lambda))$ for all $\lambda\in\Lambda$ and $\rho_{\tilde{\lambda}}[\bm{\cR}](\bar{h}_{\xi-\epsilon w}(\cdot, \lambda)) < \rho_{\tilde{\lambda}}[\bm{\cR}](\bar{h}_\xi(\cdot, \lambda))$ for some $\tilde{\lambda}\in\Lambda$.
    \label{prop:existsw}
\end{proposition}
\begin{proof}
We prove the forward direction using contradiction. Assume $h$ is C-Pareto optimal and there exists $w\in\Xi$ such that for an $\epsilon>0$, $\rho_{\lambda}[\cR](h_{\xi-\epsilon w}(\cdot, \lambda)) \leq \rho_{\lambda}[\cR](h_\xi(\cdot, \lambda))$ for all $\lambda\in\Lambda$ and $\rho_{\tilde{\lambda}}[\cR](h_{\xi-\epsilon w}(\cdot, \lambda)) < \rho_{\tilde{\lambda}}[\cR](h_\xi(\cdot, \lambda))$ for some $\tilde{\lambda}\in\Lambda$. Then $h_{\xi-\epsilon w}$ strictly dominates $h_\xi$ according to our definition of C-Pareto optimality for augmented hypothesis Def ~\ref{def:pareto-optimality-2} which contradicts that $h_\xi$ is C-Pareto optimal.
We prove the reverse direction using the contraposition. Assume $h_\xi$ is not C-Pareto optimal. Then there exists $h_{\xi'}$ that strictly dominates $h_\xi$, i.e., $\rho_{\lambda}[\cR](h_{\xi'}(\cdot, \lambda)) \leq \rho_{\lambda}[\cR](h_\xi(\cdot, \lambda))$ for all $\lambda\in\Lambda$ and $\rho_{\tilde{\lambda}}[\cR](h_{\xi'}(\cdot,\lambda)) < \rho_{\tilde{\lambda}}[\cR](h_\xi(\cdot, \lambda))$ for some $\tilde{\lambda}\in\Lambda$. Then there exists $w=\xi-\xi'$ and $\epsilon=1$ such that $h_{\xi-\epsilon w}$ strictly dominates $h_\xi$.
\end{proof}
Proposition~\ref{prop:existsw} shows that for C-Pareto optimality there must not be any direction $w\in\Xi$ for Pareto improvement. The non-existence of a direction for Pareto improvement is an if and only-if condition for C-Pareto optimality. 
\begin{proposition}
     In an $\epsilon$-neighbourhood of $\xi$ let $\rho_\xi[\bm{\cR}](h_\xi(\cdot, \lambda))$ be a smooth function of $\xi$ and the local gradient is defined as $v_\lambda(h_\xi):=\nabla_{\xi}\rho_\lambda[\cR](h_\xi(\cdot,\lambda))$. If $h_\xi$ is not pareto optimal then there exists a local pareto improvement direction $-w\in\Xi$ such that for all $\lambda\in \Lambda$ $w^{\top}v_\lambda(h_\xi)\geq 0$ and for some $\tilde{\lambda}\in\Lambda$ $w^{\top}v_{\tilde{\lambda}}(h_\xi)>0$. 
    \label{proposition:dotw}
\end{proposition}
\begin{proof}
    From Proposition \ref{prop:existsw}, when $h_\xi$ is not Pareto optimal, there exists a $w\in\Xi$ such that for an $\epsilon>0$, $h_{\xi-\epsilon w}\succ h_\xi$. Then for all $\lambda \in \Lambda$, 
    \begin{align*}
        \rho_\lambda [\cR](h_{\xi-\epsilon w}) &\leq \rho_\lambda [\cR](h_\xi)\\
        \rho_\lambda [\cR](h_\xi)-\epsilon  w^{\top}v_\lambda(h_\xi) + \epsilon^2\mathbf{R}&\leq \rho_\lambda [\cR](h_\xi)\\
        -\epsilon  w^{\top}v_\lambda(h_\xi) + \epsilon^2\mathbf{R}&\leq 0\hspace{50pt}(\mathbf{R}: \text{Remainder Higher order terms})\\
        \epsilon \mathbf{R}&\leq w^{\top}v_\lambda(h_\xi)
    \end{align*}
    Since $\rho_\lambda [\cR](h_{\xi-\epsilon w})\leq \rho_\lambda [\cR](h_\xi)$, then $w^{\top}v_\lambda(h_\xi)\geq 0$ as $\epsilon\rightarrow 0$ otherwise a contradiction would arise for sufficiently small $\epsilon$. Similarly for an $\tilde{\lambda} \in \Lambda$, since $\rho_{\tilde{\lambda}} [\cR](h_{\xi-\epsilon w})< \rho_{\tilde{\lambda}} [\cR](h_\xi)$, then  $w^{\top}v_{\tilde{\lambda}}(h_\xi)> 0$. 
\end{proof}
\cref{proposition:dotw} extends the argument from Proposition~\ref{prop:existsw} that when $h_\xi$ is not C-Pareto optimal, a direction for Pareto improvement must exist. Remark explains that a local Pareto improvement direction must align with the gradient of all objectives. Since the direction opposite to the local gradient of an objective shows us the direction of the improvement for the objective, then the direction opposite to local Pareto improvement $w\in\Xi$ must align with the local gradient if $-w\in\Xi$ improves the corresponding objective. Note that the local gradient of aggregated objective~\eqref{eq:aggregate-with-p-lambda} is 
\begin{equation}
    \nabla_\xi J_Q(\xi):=\mathbb{E}_{\lambda\sim Q}[v_\lambda(h_\xi)]
    \label{eq:local-gradient-w}
\end{equation}
Where $v_\lambda(h_\xi):=\nabla_{\xi}\rho_\lambda[\bm{\cR}](h_\xi(\cdot,\lambda))$ denotes the local gradient of $\rho_\lambda[\bm{\cR}](h_\xi(\cdot,\lambda))$. Then the choice of $Q$ such that $\nabla_\xi J_Q(\xi)$ is the the direction of local Pareto improvement is given by
\begin{proposition}\label{theorem:pareto-improvement-detailed}
For $\lambda\in\Lambda$, suppose $\xi\mapsto\rho_\lambda[\bm{\cR}](\bar{h}_\xi(\cdot,\lambda))$ is locally continuously differentiable in a neighbourhood of $\xi$. Define 
\begin{equation}
Q_t^* \in \underset{Q \in \Delta(\Lambda)}{\arg\min}\, \left\|\nabla_{\xi_{t-1}} J_Q(\xi_{t-1}) \right\|_2 \label{eq:choose-qt-appendix}
\end{equation}
and $v_t(\xi_t) = \nabla_{\xi_t} J_{Q^*_t}(\xi_t)$. Then the update $\xi_{t} \leftarrow \xi_{t-1} - \eta\cdot v_{t}(\xi_t)$ for an appropriate choice of $\eta >0$ always makes C-Pareto improvement. i.e., $-v_t(\xi_t)$ for all objectives $\rho_\lambda[\cR](h_\beta(\cdot,\lambda))$, $\lambda\in \Lambda$ such that $v_t(\xi_t)^Tv_\lambda(h_\xi)\geq ||v_t(\xi_t)||^2_2$.
\end{proposition}
\begin{proof}
We start by assuming that a given $Q^*_t$ exists then the update $\xi_{t} \leftarrow \xi_{t-1} - \eta\cdot v_{t}(\xi_t)$ performs local C-Pareto improvement. First we show that $\forall \lambda\in\Lambda$ the $v_t(\xi_t)^Tv_\lambda(h_\xi)\geq ||v_t(\xi_t)||^2_2$. For any distribution $Q\in\Delta(\Lambda)$, $v=\EE_{\lambda\sim Q}[v_\lambda(h_\xi)]-v_t(\xi_t)$. We can say that $\forall \epsilon\in[0,1]$  $v_t(\xi_t)+\epsilon v$ is essentially
\begin{align*}
    v_t(\xi_t)+\epsilon v &= v_t(\xi_t)+\epsilon(\EE_{\lambda\sim Q}[v_\lambda(h_\xi)]-v_t(\xi_t))\\
    &= (1-\epsilon) v_t(\xi_t) + \epsilon\EE_{\lambda\sim Q}[v_\lambda(h_\xi)]\\
    &= \EE_{\lambda\sim \epsilon Q+(1-\epsilon)Q_t^*}[v_\lambda(h_\xi)]
\end{align*}
Where $\epsilon Q+(1-\epsilon)Q_{t}^*$ is some other valid probability distribution. Therefore the norm of $v_t(\xi_t)+\epsilon v$ must be larger than or equal to the minimum norm obtained from ~\cref{eq:choose-qt-appendix}.
\begin{align*}
    (v_t(\xi_t)+\epsilon v)^T(v_t(\xi_t)+\epsilon v)&\geq v_t(\xi_t)^Tv_t(\xi_t)\\
    2\epsilon v_t(\xi_t)^Tv+\epsilon^2v^Tv&\geq 0\\
    \epsilon &\geq \frac{-2v_t(\xi_t)^Tv}{v^Tv}
\end{align*}
Since the above statement must be true for all $\epsilon\in(0,1]$. For $\epsilon=0$ equality must hold that $v_t(\xi_t)^Tv_t(\xi_t)=v_t(\xi_t)^Tv_t(\xi_t)$. Therefore, the lower bound from above must be less than or equal to 0.
\begin{align*}
    \frac{-2v_t(\xi_t)^Tv}{v^Tv}&\leq 0\\
    v_t(\xi_t)^Tv &\geq 0
\end{align*}
Replacing $v$ by $\EE_{\lambda\sim Q}[v_\lambda(h_\xi)]-v_t(\xi_t)$ then gives us that 
\begin{align*}
    v_t(\xi_t)^T(\EE_{\lambda\sim Q}[v_\lambda(h_\xi)]-v_t(\xi_t))&\geq 0\\
    v_t(\xi_t)^T\EE_{\lambda\sim Q}[v_\lambda(h_\xi)]&\geq v_t(\xi_t)^Tv_t(\xi_t)
\end{align*}
Thus we obtain that 
$\forall \lambda\in\Lambda$ the $v_{\lambda}(h_\xi)^Tv_t(\xi_t)\geq ||v_t(\xi_t)||_2^2$ by setting $Q$ to be dirac delta function at $\lambda$. Therefore from ~\cref{proposition:dotw} we can say that $h_{\xi_{t-1}-\eta v_t(\xi_t)}\succ h_{\xi_{t-1}}$. This makes $w\in\Xi$ the common direction for local C-Pareto improvement.
\end{proof}
Analogous to the definition~\ref{def:C-pareto-stationarity} we define C-Pareto stationarity for augmented hypothesis as 
\begin{definition}
    Let $\rho_\lambda[\cR](\bar{h}(\cdot, \lambda))$ be a smooth function of augmented hypothesis $\bar{h}$ and $v_\lambda(\bar{h}_\xi):=\nabla\rho_\lambda[\bm{\cR}](\bar{h}_\xi(\cdot, \lambda))$ be the local gradient then the augmented hypothesis is said to be C-Pareto Stationary if and only if there exists a probability density $q$ such that $\int v_\lambda(\bar{h}_\xi) \,d q(\lambda) = 0$. 
\end{definition}
Intuitively, C-Pareto Stationarity corresponds to local C-Pareto Optimality. For a single objective, C-Pareto stationarity is equivalent to the first-order derivative being zero.
Therefore, If an augmented hypothesis $h$ is C-Pareto optimal, it is C-Pareto stationary. This means that C-Pareto stationarity is a necessary condition for C-Pareto optimality. From Proposition~\ref{prop:existsw} we know that for a C-Pareto optimal point, no direction for Pareto improvement must exist, which implies that no direction for local Pareto improvement must also not exist. From theorem ~\ref{theorem:pareto-improvement} we know that a local direction for pareto improvement is $v_t(\xi_t)=\int v_\lambda(h_\xi)dQ^*_t(\lambda)$ where $Q^*_t=arg\min_{Q\in\Delta(\Lambda)}||\mathbb{E}_{\lambda\sim Q}[v_\lambda(h_\xi)]||$. Given that no direction for local C-Pareto improvement must exist implies that $v_t(\xi_t)=0$. This means that there exists a distribution $Q$ such that $\int v_\lambda(h_\xi)dQ(\lambda)=0$. 
This illustrates that C-Pareto stationarity is a necessary condition for C-Pareto optimality which intuitively illustrates that local C-Pareto optimality is necessary for C-Pareto optimality.

\subsection{Proof of Proposition~\ref{prop:regret-bound}}

\subsubsection{Useful results}

\begin{proposition}\label{prop:rejection-sampling}
    Let $X$ be a random variable taking values in $\cX$ and let $f : \cX \to \RR_+$ and $g : \cX\to\RR_+$ be non-negative functions. Define $Z = (f(X), g(X))$ and suppose that it admits a continuous density $p_Z$ with respect to the Lebesgue measure on $\RR^2$.  
    
    Let $\alpha, \beta > 0$ such that $p_Z(\alpha, \beta) > 0$ and let $Z_1, \ldots, Z_n$ be independent copies of $Z$. Then there exists $r \geq 1$ and a random subsampling operator $\pi$ such that $\pi([n]) \in 2^{\{1, \ldots, n\}}$, $|\pi([n])| \sim \operatorname{Binomial}(n, 1/r)$, and for any index $i \in \pi([n])$
    \begin{equation}
        \EE[Z_{i}] = (\alpha, \beta),
    \end{equation}
    where the expectation is taken against both the variable and the index.
\end{proposition}
\begin{proof}
    The proof consists in showing that the assumptions made are sufficient to construct a rejection sampling procedure where the proposal density is the density of $Z$ and the target density is a uniform centered over $(\alpha, \beta)$.

    Since $p_Z(\alpha, \beta) > 0$ and $p_Z$ is continuous, there exists an open neighbourhood of $(\alpha, \beta)$ where $p_Z$ is strictly positive. Therefore, there exists $\eta > 0$ such that if we define the closed rectangle
    \begin{align*}
        A_\alpha & = [\alpha - \eta/2, \alpha + \eta/2] \\
        A_\beta & = [\beta - \eta/2, \beta + \eta/2] \\
        A & = A_\alpha\times A_\beta,
    \end{align*}
    then $p_Z(x, x') > 0$ for any $(x, x')\in A$ and admits a positive lower bound on $A$. Further, we can define the uniform random variable $U\sim\operatorname{Uniform}(A)$ with probability density
    \begin{equation*}
        p_U(x, x') = \frac{1}{\eta^2}, \enspace \forall (x, x')\in A.
    \end{equation*}

    Then, by upper boundedness of $p_U$ over $A$ and lower-boundedness of $p_Z$ over $A$, there exists $r \geq 1$ such that for any $(x, x')\in A$,
    \begin{equation*}
        \frac{p_U(x, x')}{p_Z(x, x')}\leq r.
    \end{equation*}

    As a result, we can formally construct a rejection sampling procedure to sample from $U$ using samples from $Z$ with acceptance rate $1/r$. It is important to note this is only a formal construction to show the existence of an appropriate subsampling procedure. In practice, we may not be able to evaluate $p_Z$ and therefore may be unable to effectively implement the procedure.

    \begin{algorithm}[h]
    \caption{Algorithmic definition of the random subsampling operator $\pi$}
    \label{alg:rejection-sampling}
    \begin{algorithmic}[1]
        \STATE \textbf{Input:} $p_U$, $p_Z$, $r$, $Z_1, \ldots, Z_n$
        \STATE Initialise \texttt{subsampled} = $\{\}$
        \FOR{$i \in \{1, \ldots, n\}$}
            \STATE Let $U_i \sim \operatorname{Uniform}([0, 1])$
            \IF{ $U_i \leq p_U(Z_i)/rp_Z(Z_i)$}
                \STATE Append $i$ to \texttt{subsampled}
            \ENDIF
        \ENDFOR
        \STATE \textbf{Return} \texttt{subsampled}
    \end{algorithmic}
    \end{algorithm}

Algorithm~\ref{alg:rejection-sampling} outlines an algorithmic definition of a random subsampling operator $\pi : 2^{[n]} \to 2^{[n]}$ based on rejection sampling. We emphasise the random nature of the operator $\pi$ as $Z_1, \ldots, Z_n$ are treated throughout as random variables. By property of rejection sampling, the number of accepted samples $|\pi([n])|$ or $|$\texttt{subsampled}$|$ follows a Binomial distribution with $n$ trials and probability of success $1/r$. Finally, we have by construction that for any $i\in \pi([n])$
\begin{equation*}
    \EE[Z_i] = \EE[\operatorname{Uniform}(A)] = (\alpha, \beta) 
\end{equation*}
which concludes the proof.
\end{proof}

\subsubsection{Proof of the main result}
We begin by introducing notations which will be used in this proof. Suppose we observe $n\in\NN$ of IID observations from each environment, i.e., we observe $(x_1^{(i)}, y_1^{(i)}), \ldots, (x_n^{(i)}, y_n^{(i)}) \sim \PP_i$ for every $i\in\{1, \ldots, d\}$. Furthermore, let $Q\in\Delta(\Lambda)$ be the scalarisation density the learner chooses and let $\lambda_1, \ldots, \lambda_m\sim Q$ be independent samples from this distribution.

For each environment $i \in \{1, \ldots, d\}$, we define an empirical risk
\begin{equation*}
    \hat\cR_i(f) = \frac{1}{n} \sum_{j=1}^n \ell(y_j^{(i)}, f(x_j^{(i)})), \enspace f\in\cH,
\end{equation*}
which we concatenate into an empirical risk profile $\bm{\hat\cR} = (\hat\cR_1, \ldots, \hat\cR_d)$. We can easily verify that for any $i\in\{1, \ldots, d\}$, $\EE[\hat \cR_i] = \cR_i$ where the expectation is taken against $\PP_i$, thus $\EE[\bm{\hat\cR}] = \bm{\cR}$. Therefore, if we take the empirical aggregated risk to be $\rho_\lambda[\bm{\hat \cR}]$ for $\lambda\in\Lambda$ and assume that $\rho_\lambda : \cL_2^d(\cH) \to \cL_2(\cH)$ is a linear risk aggregation function, it follows that
\begin{equation*}
    \EE\left[\rho_\lambda[\bm{\hat \cR}]\right] = \rho_\lambda\left[\EE[\bm{\hat \cR}]\right] = \rho_\lambda[\bm{\cR}].
\end{equation*}

Finally, define the empirical scalarised risk using the values $\lambda_1, \ldots, \lambda_m$ sampled above, for $g\in\cH_\Lambda$ as
\begin{equation*}
    \hat J_Q(g) = \frac{1}{m} \sum_{i=1}^m \rho_{\lambda_i}[\bm{\hat\cR}](g(\cdot, \lambda_i)).
\end{equation*}

In what follows, we will always assume there exists a function $\hat h\in\cH_\Lambda$ such that for any $\lambda\in\Lambda$, $\hat h(\cdot, \lambda)$ is a minimiser of the empirical aggregated risk $\rho_\lambda[\bm{\hat \cR}]$, i.e.,
\begin{equation*}
    \hat h(\cdot, \lambda) \in \underset{f\in\cH}{\arg\min}\,\rho_\lambda[\bm{\hat\cR}](f)\,,\enspace\forall\lambda\in\Lambda,
\end{equation*}
and that the empirical scalarised risk also admits a minimiser which we denote $\hat g\in\cH_\Lambda$, i.e.,
\begin{equation*}
    \hat g \in \underset{g\in\cH_\Lambda}{\arg\min}\, \hat J_Q(g).
\end{equation*}

The following lemma shows that when such minimisers exists, then $\hat g(\cdot, \lambda_i)$ is automatically a minimiser of the empirical aggregated risk $\rho_{\lambda_i}[\bm{\hat\cR}]$.
\begin{lemma}\label{lemma:AAA}
    Suppose there exists $\hat h, \hat g$ defined as above. Then $\hat g(\cdot, \lambda_i)$ minimises $\rho_{\lambda_i}[\bm{\hat{\cR}}]$ for all $i\in\{1, \ldots, m\}$.
\end{lemma}
\begin{proof}
    Let $\hat h\in\cH_\Lambda$ such that $\hat h(\cdot, \lambda)\in\underset{f\in\cH}{\arg\min}\, \rho_{\lambda}[\bm{\hat\cR}](f)$ for any $\lambda\in\Lambda$. Then, we have
    \begin{align*}
        & \rho_{\lambda}[\bm{\hat\cR}](\hat g(\cdot, \lambda)) \geq  \rho_{\lambda}[\bm{\hat\cR}](\hat h(\cdot, \lambda))\,,\enspace \forall\lambda\in\Lambda \\
       \Rightarrow \enspace & \rho_{\lambda_i}[\bm{\hat\cR}](\hat g(\cdot, \lambda_i)) \geq  \rho_{\lambda_i}[\bm{\hat\cR}](\hat h(\cdot, \lambda_i))\,,\enspace \forall i\in\{1, \ldots, m\} \\
       \Rightarrow\enspace & \hat J_Q(\hat g) \geq \hat J_Q(\hat h)  \\
       \Rightarrow\enspace & \hat J_Q(\hat g) = \hat J_Q(\hat h) & (\hat g\in\arg\min \hat J_Q) \\
       \Rightarrow\enspace & \frac{1}{m}\sum_{i=1}^m \rho_{\lambda_i}[\bm{\hat\cR}](\hat g(\cdot, \lambda_i)) - \rho_{\lambda_i}[\bm{\hat\cR}](\hat h(\cdot, \lambda_i)) = 0 \\
       \Rightarrow\enspace & \rho_{\lambda_i}[\bm{\hat\cR}](\hat g(\cdot, \lambda_i)) = \rho_{\lambda_i}[\bm{\hat\cR}](\hat h(\cdot, \lambda_i))\,,\enspace\forall i \in\{1, \ldots, m\} & (\text{sum of positives}) \\
       \Rightarrow\enspace & \hat g(\cdot, \lambda_i)\in\underset{f\in\cH}{\arg\min}\, \rho_{\lambda_i}[\bm{\hat\cR}](f)\,,\enspace \forall i\in\{1, \ldots, m\}.
    \end{align*}
    This concludes the proof.
\end{proof}

Finally, before we turn to the main result, recall that we assume there exists $h^*\in\cH_\Lambda$ such that $h^*(\cdot, \lambda)\in\cH$ is a Bayes optimal model for any $\lambda\in\Lambda$, i.e., $h^*(\cdot, \lambda) \in \underset{f\in\cH}{\arg\min}\, \rho_\lambda[\bm{\cR}](f).$ For any $\lambda\in\Lambda$, we denote the resulting Bayes risk as
\begin{equation*}
    \rho_{\lambda}[\bm{\cR}]^\star = \rho_\lambda[\bm{\cR}](h^*(\cdot, \lambda)).
\end{equation*}

Let $\lambda_\text{op}\in\Lambda$ be the choice of $\lambda$ which reflects the operator's preference, but is unknown to the learner. The following result provides a bound on the excess risk at $\lambda_\text{op}$ when using $\hat g$ as a hypothesis.

\begin{proposition}
    Let $Q\in\Delta(\Lambda)$ and let $\lambda_\text{op}\in\Lambda$ such that $Q(\lambda_\text{op}) > 0$. Suppose that $\rho_\lambda$ is a linear, idempotent aggregation operator and that the loss $\ell$ is upper bounded by $M \geq 0$. Then there exists $q\in (0,1)$ such that for any $\delta > q^m$, the following inequality holds with probability $1 - \delta$:
    \begin{equation*}
        \big|\rho_{\lambda_\text{op}}[\bm{\cR}](\hat g(\cdot, \lambda_\text{op})) - \rho_{\lambda_\text{op}}[\bm{\cR}]^\star\big | \leq 2M \left(\sqrt{\frac{\log(2/\eta_\delta)}{2n}} + \sqrt{\frac{\log(2/\eta_\delta)}{2 m(1-q)(1-q^m)}}\right),
    \end{equation*}
    where $\eta_\delta = (\delta - q^m)/(1 - q^m)$.
\end{proposition}

\begin{proof}
The proof consists in (1) constructing a subsequence from $\lambda_1, \ldots, \lambda_m$ such that the empirical scalarised risks converge to appropriate limits, (2) using these subsequences to apply concentration inequalities to the excess risk when the subsequence exists and (3) combining the results together in the general case.
\strut\newline

    \textbf{(1) -- Constructing an appropriate subsampling procedure}

    Let $\lambda$ be a random variable with probability density function $Q$ over $\Lambda$. It induces a real-valued distribution over the risks $\rho_{\lambda}[\bm{\hat\cR}](\hat g(\cdot, \lambda))$ and $\rho_{\lambda}[\bm{\hat\cR}](\hat h(\cdot, \lambda))$. We assume that $\left(\rho_{\lambda}[\bm{\hat\cR}](\hat g(\cdot, \lambda)), \rho_{\lambda}[\bm{\hat\cR}](\hat h(\cdot, \lambda))\right)$ admits a continuous density with respect to the Lebesgue measure in $\RR^2$ we denote $p$. Further, define
    \begin{align*}
        \alpha_\text{op} & = \rho_{\lambda_\text{op}}[\bm{\hat\cR}](\hat g(\cdot, \lambda_\text{op})) \\
        \beta_\text{op} & = \rho_{\lambda_\text{op}}[\bm{\hat\cR}](\hat h(\cdot, \lambda_\text{op})).
    \end{align*}
    Since $Q(\lambda_\text{op}) > 0$, we have $p(\alpha_\text{op}, \beta_\text{op}) > 0$. Then by Proposition~\ref{prop:rejection-sampling}, given $\lambda_1, \ldots, \lambda_m\sim Q(\lambda)$ IID, there exists $r \geq 1$ and a random subsampling $\pi([m])\in 2^{[m]}$ such that for any index $i\in\pi([m])$ we have
    \begin{equation*}
        \EE\left[\left(\rho_{\lambda_i}[\bm{\hat\cR}](\hat g(\cdot, \lambda_i)), \rho_{\lambda_i}[\bm{\hat\cR}](\hat h(\cdot, \lambda_i))\right)\right] = (\alpha_\text{op}, \beta_\text{op}).
    \end{equation*}

    In particular, let $p = |\pi([m])| \sim \operatorname{Binomial}(m, 1/r)$ denote the number of subsampled elements and assume without loss of generality these are the first $p$ ones. Then, conditionally on $p \geq 1$, we have that $\frac{1}{p}\sum_{i=1}^p \rho_{\lambda_i}[\bm{\hat\cR}](\hat g(\cdot, \lambda_i))$ and $\frac{1}{p}\sum_{i=1}^p \rho_{\lambda_i}[\bm{\hat\cR}](\hat h(\cdot, \lambda_i))$ are respectively unbiased estimators of $\alpha_\text{op} = \rho_{\lambda_\text{op}}[\bm{\hat\cR}](\hat g(\cdot, \lambda_\text{op}))$ and $\beta_\text{op} = \rho_{\lambda_\text{op}}[\bm{\hat\cR}](\hat h(\cdot, \lambda_\text{op}))$.

    \strut\newline
    \textbf{(2.1) -- Bounding the regret when $p\geq 1$ is fixed}

    Suppose we are in a fixed setting where $p\geq 1$. Then we can decompose and upper bound the regret following
    \begin{align*}
        \rho_{\lambda_\text{op}}[\bm{\cR}](\hat g(\cdot, \lambda_\text{op})) - \rho_{\lambda_\text{op}}[\bm{\cR}]^\star & = \rho_{\lambda_\text{op}}[\bm{\cR}](\hat g(\cdot, \lambda_\text{op})) - \rho_{\lambda_\text{op}}[\bm{\hat\cR}](\hat g(\cdot, \lambda_\text{op})) \\
        & \quad + \rho_{\lambda_\text{op}}[\bm{\hat\cR}](\hat g(\cdot, \lambda_\text{op})) - \frac{1}{p}\sum_{i=1}^p \rho_{\lambda_i}[\bm{\hat\cR}](\hat g(\cdot, \lambda_i)) \\
        & \quad + \frac{1}{p}\sum_{i=1}^p \rho_{\lambda_i}[\bm{\hat\cR}](\hat g(\cdot, \lambda_i)) - \rho_{\lambda_\text{op}}[\bm{\hat\cR}](\hat h(\cdot, \lambda_\text{op}))  \\
        & \quad + \rho_{\lambda_\text{op}}[\bm{\hat\cR}](\hat h(\cdot, \lambda_\text{op})) - \rho_{\lambda_\text{op}}[\bm{\cR}]^\star \\
        \\
        & \leq \rho_{\lambda_\text{op}}[\bm{\cR}](\hat g(\cdot, \lambda_\text{op})) - \rho_{\lambda_\text{op}}[\bm{\hat\cR}](\hat g(\cdot, \lambda_\text{op})) \\
        & \quad + \rho_{\lambda_\text{op}}[\bm{\hat\cR}](\hat g(\cdot, \lambda_\text{op})) - \frac{1}{p}\sum_{i=1}^p \rho_{\lambda_i}[\bm{\hat\cR}](\hat g(\cdot, \lambda_i)) \\
        & \quad + \frac{1}{p}\sum_{i=1}^p \rho_{\lambda_i}[\bm{\hat\cR}](\hat h(\cdot, \lambda_i)) - \rho_{\lambda_\text{op}}[\bm{\hat\cR}](\hat h(\cdot, \lambda_\text{op}))  & (\text{Lemma~\ref{lemma:AAA}})\\
        & \quad + \rho_{\lambda_\text{op}}[\bm{\hat\cR}](h^*(\cdot, \lambda_\text{op})) - \rho_{\lambda_\text{op}}[\bm{\cR}]^\star & \big(\hat h(\cdot, \lambda_\text{op})\in\arg\min \rho_{\lambda_\text{op}}[\bm{\hat\cR}]\big)\\
        \\
        & \leq 2\sup_{f\in\cH}\left|\rho_{\lambda_\text{op}}[\bm{\cR}](f) - \rho_{\lambda_\text{op}}[\bm{\hat\cR}](f)\right| \\
        & \quad + \left|\rho_{\lambda_\text{op}}[\bm{\hat\cR}](\hat g(\cdot, \lambda_\text{op})) - \frac{1}{p}\sum_{i=1}^p \rho_{\lambda_i}[\bm{\hat\cR}](\hat g(\cdot, \lambda_i))\right| \\
        & \quad + \left|\frac{1}{p}\sum_{i=1}^p \rho_{\lambda_i}[\bm{\hat\cR}](\hat h(\cdot, \lambda_i)) - \rho_{\lambda_\text{op}}[\bm{\hat\cR}](\hat h(\cdot, \lambda_\text{op}))\right| \\
    \end{align*}

    Let $\eta \in (0, 1)$ fixed. By linearity of $\rho_\lambda$, we have shown that $\rho_{\lambda}[\bm{\hat\cR}](f)$ is an unbiased estimator of $\rho_{\lambda}[\bm{\cR}](f)$. Therefore, McDiarmid's inequality gives us that we have with probability at least $1 - \eta/3$
    \begin{equation*}
        \left|\rho_{\lambda_\text{op}}[\bm{\cR}](f) - \rho_{\lambda_\text{op}}[\bm{\hat\cR}](f)\right| \leq M \sqrt{\frac{\log(6/\eta)}{2n}}.
    \end{equation*}

    If we denote $Z_{\lambda_i} = \rho_{\lambda_i}[\bm{\hat\cR}](\hat g(\cdot, \lambda_i))$ for the $p$ accepted samples from the rejection sampling procedure, then we have by construction that $\frac{1}{p} \sum_{i=1}^p Z_{\lambda_i}$ is an unbiased estimator of $\rho_{\lambda_\text{op}}[\bm{\hat\cR}](\hat g(\cdot, \lambda_\text{op}))$. Therefore, we can also apply McDiarmid's inequality to obtain that with probability at least $1-\eta/3$
    \begin{equation*}
        \left|\rho_{\lambda_\text{op}}[\bm{\hat\cR}](\hat g(\cdot, \lambda_\text{op})) - \frac{1}{p}\sum_{i=1}^p \rho_{\lambda_i}[\bm{\hat\cR}](\hat g(\cdot, \lambda_i))\right| \leq M\sqrt{\frac{\log(6/\eta)}{2p}}.
    \end{equation*}
    
    Applying the same reasoning to the last line and combining the bounds together using the union bound we get that with probability at least $1 - \eta$
    \begin{equation*}
         \rho_{\lambda_\text{op}}[\bm{\cR}](\hat g(\cdot, \lambda_\text{op})) - \rho_{\lambda_\text{op}}[\bm{\cR}]^\star \leq 2M\sqrt{\frac{\log(6/\eta)}{2n}} + 2M\sqrt{\frac{\log(6/\eta)}{2p}}.
    \end{equation*}

    \strut\newline
    \textbf{(2.2) -- Integrating the upper bound against $p$ given $p \geq 1$}

    We now consider the random setting, conditional on $p \geq 1$. Recall that the number of accepted samples from the rejection sampling procedure $p$ follows a $\operatorname{Binomial}(n, 1/r)$ distribution. We want to take the expectation of the established probabilistic upper bound with respect to $p$ given that $p \geq 1$. Let $q = 1 - 1/r$ denote the rejection rate, then we have for any $k \geq 1$
    \begin{equation*}
        \PP(p=k\mid p\geq 1) = \frac{1}{1 - q^m} \begin{pmatrix}m \\ k\end{pmatrix}(1 - q)^k q^{m-k}.
    \end{equation*}

    This corresponds to a positive Bernoulli distribution~\citep{grab1954tables}, and in particular if we denote $B_{m, r}(k) = \PP(\operatorname{Bernoulli}(m, 1/r) \leq k)$, we have from \citep{grab1954tables} Eq. (12) that
    \begin{align*}
        \EE\left[\frac{1}{p}\mid p\geq 1\right] & \leq \frac{1}{(m+1)(1-q)(1-q^m)}\left(\left[1 - B_{m+1, r}(1)\right] + \frac{3}{(1-q)(m+2)} \left[1 - B_{m+2, r}(2)\right]\right) \\
        & \leq \frac{1}{m(1-q)(1-q^m)}.
    \end{align*}

    Therefore, it follows that
    \begin{align*}
        \EE\left[\sqrt{\frac{\log(6/\eta)}{2p}} \mid p\geq 1\right] & = \EE\left[\sqrt{\frac{\log(6/\eta)1/p}{2}} \mid p\geq 1\right] \\
        & \leq \sqrt{\frac{\log(6/\eta)\EE[1/p \mid p\geq 1]}{2}} & \text{(Jensen)} \\
        & \leq \sqrt{\frac{\log(6/\eta)}{2 m (1-q)(1-q^m)}},
    \end{align*}
    and by applying this to the probabilistic upper bound on the excess risk we have obtained earlier, we get that with probability at least $1 - \eta$
    \begin{equation*}
         \rho_{\lambda_\text{op}}[\bm{\cR}](\hat g(\cdot, \lambda_\text{op})) - \rho_{\lambda_\text{op}}[\bm{\cR}]^\star \leq 2M\sqrt{\frac{\log(6/\eta)}{2n}} + 2M\sqrt{\frac{\log(6/\eta)}{2 m (1-q)(1-q^m)}}.
    \end{equation*}

    \strut\newline
    \textbf{(3) -- Combining things together}

    Now that we have established a probabilistic upper-bound on the excess risk when at least one sample is accepted by $\pi$, we set out to obtain a general probabilistic bound on the excess risk. Let $q = 1- 1/r$ be the rejection rate of the rejection sampling procedure and fix $\delta \in (q^m, 1)$.

    Take $\eta_\delta = (\delta - q^m)/(1 - q^m)$ and $\varepsilon_\delta = 2M\sqrt{\frac{\log(6/\eta_\delta)}{2n}} + 2M\sqrt{\frac{\log(6/\eta_\delta)}{2 m (1-q)(1-q^m)}}$, then we have
    \begin{align*}
        \PP\left(\rho_{\lambda_\text{op}}[\bm{\cR}](\hat g(\cdot, \lambda_\text{op})\right) - \rho_{\lambda_\text{op}}[\bm{\cR}]^\star > \varepsilon_\delta) & = \underbrace{\PP\left(\rho_{\lambda_\text{op}}[\bm{\cR}](\hat g(\cdot, \lambda_\text{op})) - \rho_{\lambda_\text{op}}[\bm{\cR}]^\star > \varepsilon_\delta \mid p = 0\right)}_{\leq 1}\underbrace{\PP\left(p = 0\right)}_{=q^m} \\
        & + \PP\left(\rho_{\lambda_\text{op}}[\bm{\cR}](\hat g(\cdot, \lambda_\text{op})) - \rho_{\lambda_\text{op}}[\bm{\cR}]^\star > \varepsilon_\delta \mid p\geq 1\right)\PP\left(p\geq 1\right) \\
        \\
        & \leq q^m + (1 - q^m)\PP\left(\rho_{\lambda_\text{op}}[\bm{\cR}](\hat g(\cdot, \lambda_\text{op})) - \rho_{\lambda_\text{op}}[\bm{\cR}]^\star > \varepsilon_\delta \mid p\geq 1\right) \\
        \\
        & \leq q^m + (1 - q^m)\eta_\delta \\
        & = q^m + (1 - q^m)\frac{\delta - q^m}{1 - q^m } = \delta,
    \end{align*}
    where the last derivations follow from the construction of $\varepsilon_\delta$ and $\eta_\delta$. This shows that for any $\delta \in (q^m, 1)$, the following inequality holds with probability $1 - \delta$
    \begin{equation*}
        \rho_{\lambda_\text{op}}[\bm{\cR}](\hat g(\cdot, \lambda_\text{op}))- \rho_{\lambda_\text{op}}[\bm{\cR}]^\star \leq 2M \sqrt{\frac{\log(6/\eta_\delta)}{2n}} + 2M\sqrt{\frac{\log(6/\eta_\delta)}{2 m (1-q)(1-q^m)}},
    \end{equation*}
    where $\eta_\delta = (\delta - q^m) / (1 - q^m)$. This concludes the proof.

\end{proof}

\section{Conditional Value-at-Risk (CVaR)}
\begin{proposition}
    Let $I =\{1, \dots, m\}$ be an index set and $R: I \to \RR_+$ such that $R(i) = \hat{R}_i$ for $i\in I$. Denote $C(I)$ as the space of real-valued, continuous function on $I$ and $C(I)^*$ its dual, i.e., $\{T: C(I) \to \RR \}$. Then there is a finite measure $\mu$ on $I$ such that for any $T \in C(I)^*$ and $R \in C(I)$, we have
\begin{align*}
    T(R) = \sum_{i\in I} R_i \mu_i .
\end{align*}
\label{prop:cvar_aggregate}
\end{proposition}

\begin{proof}[Sketch Proof]
    The key is to notice $I$ is a compact metric space because it is bounded. Furthermore, all functions on discrete space are automatically continuous. This allows us to directly apply the Riesz-Markov-Kakuani representation theorem.
\end{proof}

The proposition implies that no matter how we aggregate a risk profile, it will always correspond to some kind of weighted average. From the perspective of optimisation, since these weights are always convex (noramlising the weights does not change the optimisation), it can then be understood that whenever we aggregate the risk profile, we are picking a particular weighted distribution to perform the standard ERM.

\section{Single-Domain Scenario}
\label{sec:single-domain}

In a single-domain setting, we envision two possible approaches to imprecise learning. The first approach treats each training data point as an individual domain, estimating the risk profile through point-wise loss functions, denoted as $\bm{\cR}(f) = (\ell(f(x_1),y_1),\ldots,\ell(f(x_n),y_n))$. The second approach delineates a credal set by an $\epsilon$-ball around the empirical distribution of the training data, akin to Distributionally Robust Optimisation (DRO). Subsequently, it extracts a finite number of extreme points from this credal set, which then represent the risk profile. While the first approach can be directly implemented within the current framework, the second approach entails a non-trivial extension of the existing setup.

\section{Risk Profiles of Simulation}
\paragraph{Simulation of Risk Profile:}
In economic theory, risk aversion explains the inclination to accept a situation with a more predictable but possibly lower payoff than another situation with a very unpredictable but possibly higher payoff. In OOD research, the term risk averseness has been conceptually used to describe the operator's risk perception for the model's risk profile (i.e., the distribution of $\hat{R}$). A risk-averse operator prefers models whose risk is more predictable but possibly higher than models whose risk is less predictable but possibly lower. Given that the operator at test time have a risk averseness between "less risk averse" and "risk averse" and by having $h(x, \lambda)$ we can cover this spectrum of the operator's risk averseness. Given that we use CVaR, the entire spectrum of an ML Operators potential risk averseness is encoded in the interval of $\lambda$ being between 0 and 1. By construction, $h(x,\lambda)$ can cover the spectrum of "risk averseness" because it corresponds to the prediction function we obtain at $CVaR(\lambda)$. Hence, we verify this hypothesis.
\paragraph{Experiment 1A:} Assume a linear model $Y_{e}=\theta_{e}X+\epsilon$, where $X \sim \mathcal{N}(2,0.2)$ and $\epsilon\sim \mathcal{N}(0,0.1)$. We simulate different environments by drawing $\theta$ from a Beta distribution $Beta(0.1,0.2)$. In total, we generate for 250 train and test domains 100 observations each.


\begin{figure*}[h!]
    \caption{\cref{fig:1d_linear_beta_data} illustrated the data and the ideal learner $f_{\lambda}(\hat{\theta})\in\cH$ for $\lambda \in \{0.05, \dots, 0.95\}$. \cref{fig:1d_linear_beta_landscape} describes the landscape of the objective function $\rho$ (CVaR) for the ideal learner. We plot $\hat{\theta}$ as circles.\cref{fig:1d_linear_beta_risk} describes the Risk profile for $\lambda \in \{0.05, \dots, 0.95\}$ for the ideal learner. \cref{fig:1d_linear_beta_risk_iro} describes the Risk profile for $\lambda \in \{0.05, \dots, 0.95\}$ Imprecise Learner.}
     \centering
     \begin{subfigure}[b]{0.24\textwidth}
         \centering
         \includegraphics[width=\textwidth]{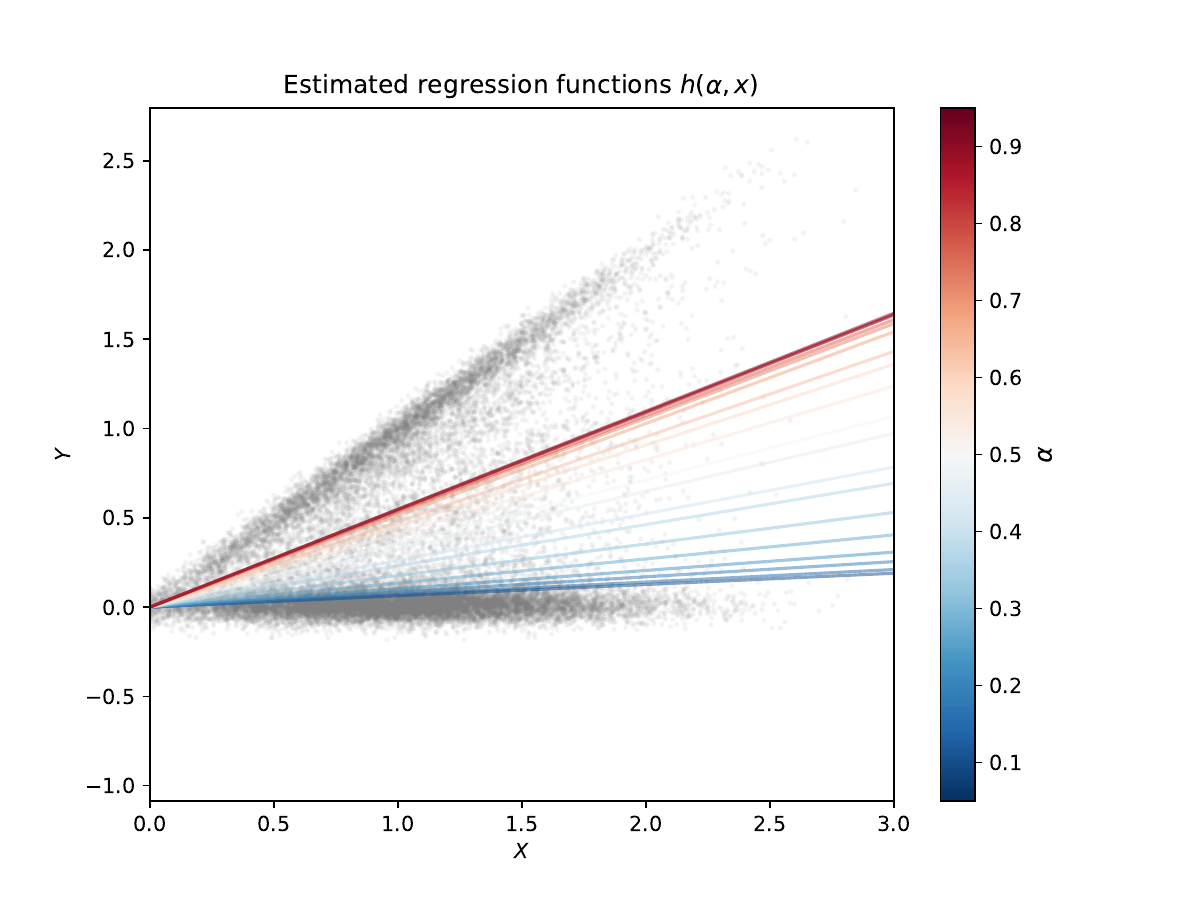}
         \caption{}
         \label{fig:1d_linear_beta_data}
     \end{subfigure}
     \begin{subfigure}[b]{0.24\textwidth}
         \centering
         \includegraphics[width=\textwidth]{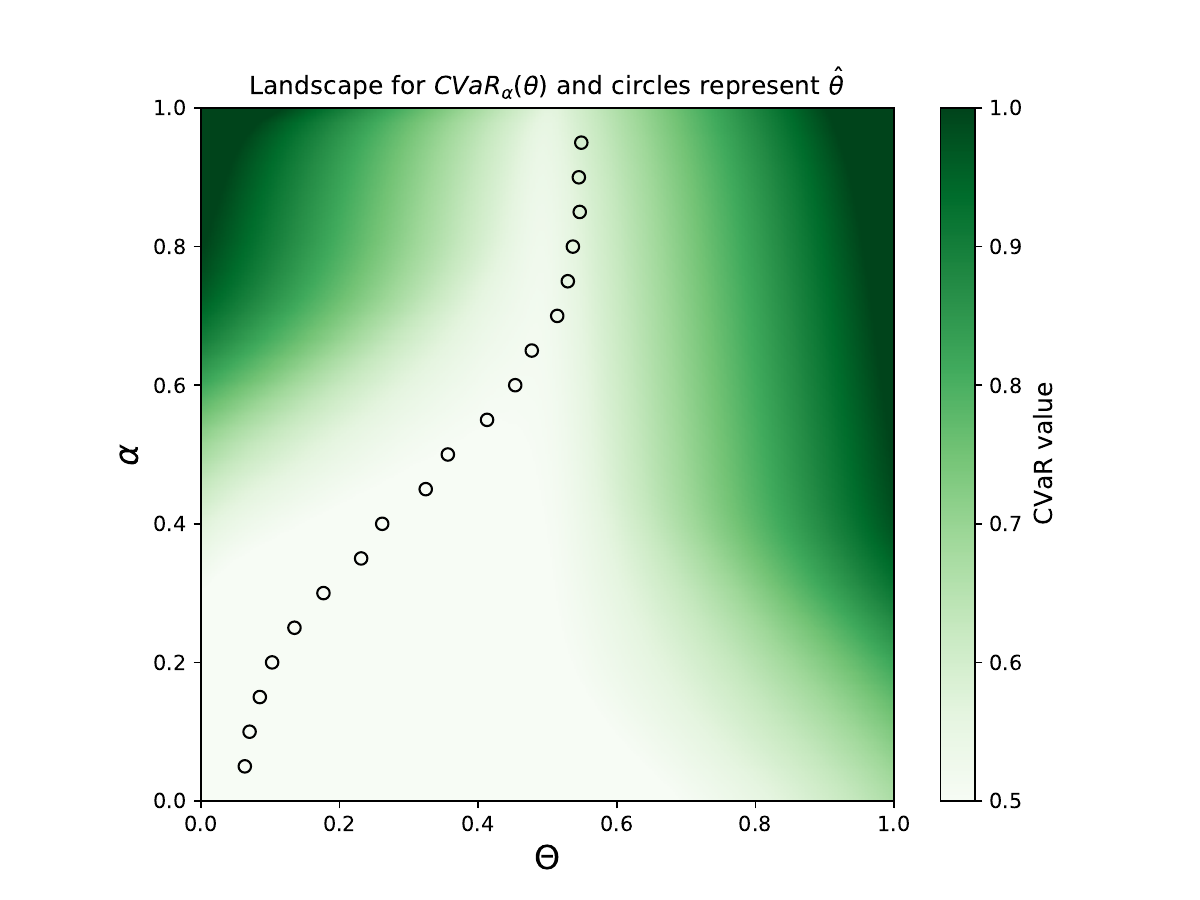}
         \caption{}
         \label{fig:1d_linear_beta_landscape}
     \end{subfigure}
     \begin{subfigure}[b]{0.24\textwidth}
         \centering
         \includegraphics[width=\textwidth]{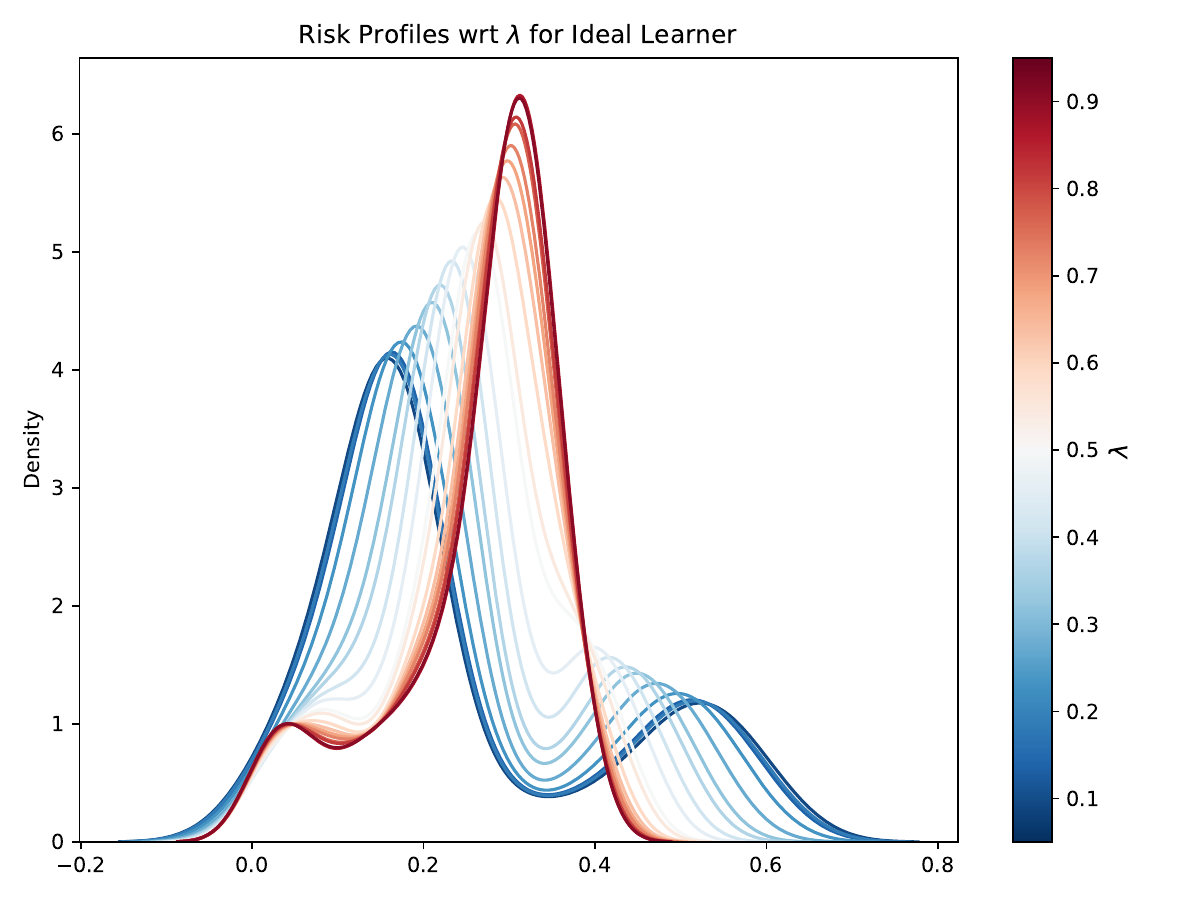}
         \caption{}
         \label{fig:1d_linear_beta_risk}
     \end{subfigure}
     \begin{subfigure}[b]{0.24\textwidth}
         \centering
         \includegraphics[width=\textwidth]{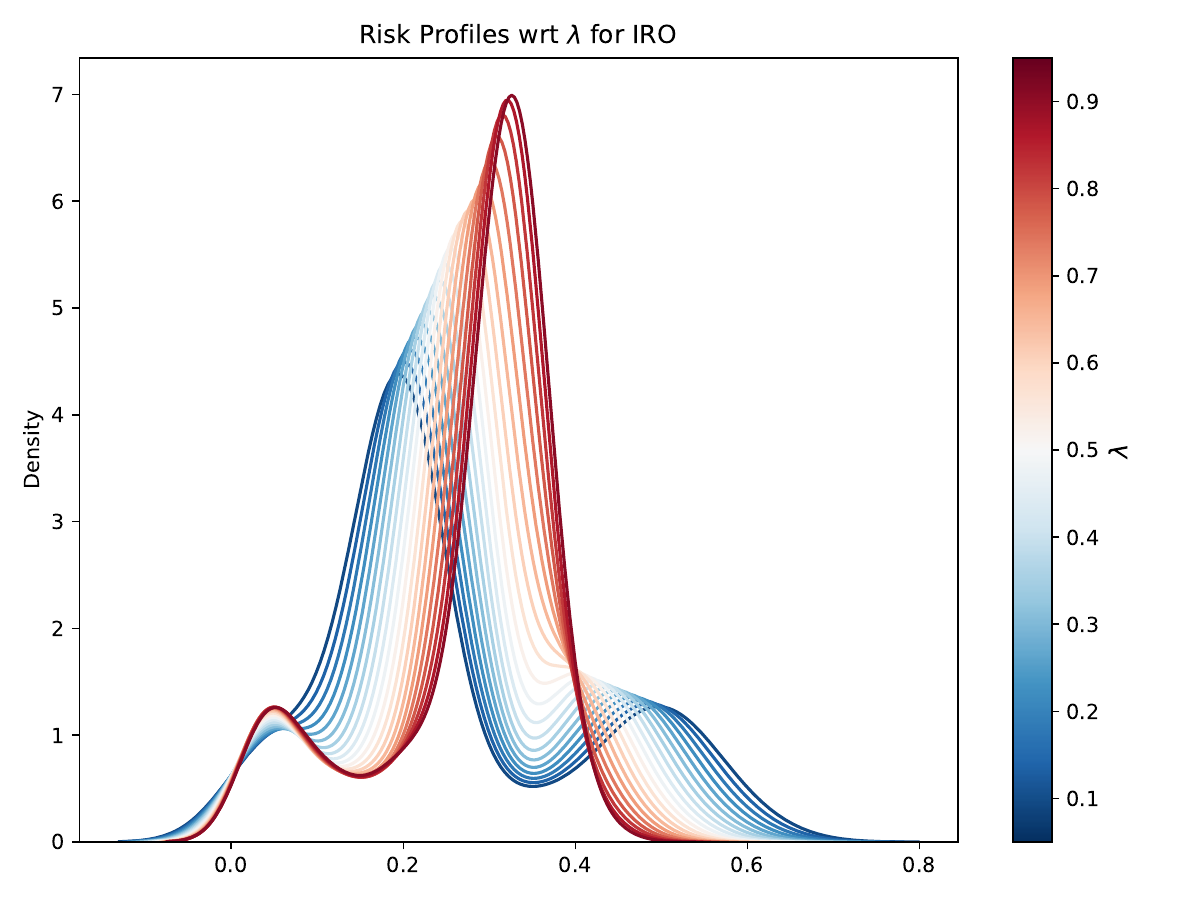}
         \caption{}
         \label{fig:1d_linear_beta_risk_iro}
     \end{subfigure}
\end{figure*}

Each data line corresponds to a domain in Figure \ref{fig:1d_linear_beta_data}. Hence the domains differ in their slope. Since we take $\theta$ from the bimodal distribution $Beta(0.1,0.2)$, we observe that the domains form two clusters. The more dominant cluster includes the domains with smaller $\theta$. Subsequently, we aim to find the optimal $\hat{\theta}$ for all $\lambda \in \{0.05, \dots, 0.95\}$ by solving the corresponding CVaR objective. As we can see from this plot, the optimal lines for small values of $\lambda$ cluster around the dominant cluster of the environments. We consider the dark blue line ($\lambda$=0.05) as the ``average case''. When increasing $\lambda$, the lines get closer to the second cluster of domains, which could be considered as the "worst-case". Hence, the dark red line ($\lambda$=0.95) could be somewhat considered to be the estimated $\theta$ that works well in the worst cases.

In Figure \ref{fig:1d_linear_beta_landscape}, we observe that for higher values of $\lambda$, the optimal solutions for $\theta$ vary a lot, while for smaller values of $\lambda$, the optimal solutions for $\theta$ do not vary significantly. As an interpretation, we can say that it is likely that the problem becomes harder for small $\lambda$. This interpretation is supported by the fact that when we choose $\lambda$ to be high, we condition on the tail of the $\bm{\cR}$, thus considering only a subset of the domains (i.e., lesser data for optimization). When looking at the optimal $\hat{\theta}$, they form a smooth curve across all $\lambda \in \{0.05, \dots, 0.95\}$.

Lastly, in Figure \ref{fig:1d_linear_beta_risk}, we see how the distribution of the risk changes across all $\lambda$. As expected, when choosing higher $\lambda$ we consider higher risks from the risk profile $\bm{\cR}$ and minimize these parts in the optimization. We observe that for higher values of $\lambda$ the risk profile does not transition smoothly contrary to the case of IRO in Figure \ref{fig:1d_linear_beta_risk_iro}. We postulate that this is because an ideal learner essentially throws away the data from low-risk domains when focusing on high-risk domains due to the formulation of CVaR as an aggregator. However, since IRO learns all the objectives simultaneously it can implicitly address this issue of a finite number of domains for training in $\lambda$ corresponding to higher risks. This observation is consistent with our observation from real-world experiments on UCI bike rentals in ~\cref{fig:bike-rentals}.
\section{Experiments on CMNIST}
\label{appendix:cmnist}
\subsection{Dataset Setup}
We conduct a large-scale experiment using an extension of the CMNIST dataset~\cite {arjovski2021}. The CMNIST comprises data from the MNIST dataset modified to the task of binary classification. For the standard task in CMNIST, the digits (0-4) and (5-9) have to be classified into two labels 0 and 1. Another feature as color is introduced in the training domain where digits are colored red or green such that the color is predictive of the true label e.g. domain 0.3 i.e. $P(Y=1\,|\,\text{color}=\text{red})=0.3$ and $P(Y=0\,|\,\text{color}=\text{red})=0.7$. Whereas for domain 0.9 it would mean $P(Y=1\,|\,\text{color}=\text{red})=0.9$ and $P(Y=0\,|\,\text{color}=\text{red})=0.1$. That is the mechanism by which color influences the label changes across domains. However, shape has a stable mechanism of prediction across domains i.e. $P(Y=0\,|\,\text{shape}\in\{0,1,\ldots,4\})=0.75$ and $P(Y=1\,|\,\text{shape}\in\{5,6,\ldots,9\})=0.75$. 
\subsection{Experimental Setup and Baselines}
We consider a scenario where we sample environments from a long-tail distribution at training time to model data collection in the real world, such as low-resource languages. We sample 10 training environments from a Beta(0.9,1) distribution exactly $\{0.01, 0.02,  0.05, 0.07, 0.09, 0.12, 0.14, 0.58, 0.7, 0.99\}$. However, we do not assume IID distribution on environments, i.e. at test time we evaluate all the environments $\{0.0,0.1,\dots,0.9,1.0\}$. Each environment is assumed to be influenced by both color and shape where the mechanism of color's influence changes but shape affects the target stably. 
\begin{figure*}

    \begin{subfigure}[b]{0.45\textwidth}
    \centering
    \includegraphics[width=0.5\textwidth]{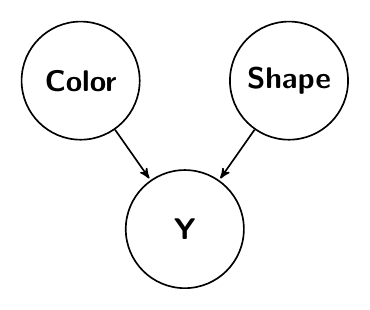}
    \caption{DAG of features and target in CMNIST}
    \label{fig:dag_of_cmnist}
    \end{subfigure}%
   \begin{subfigure}[b]{0.45\textwidth}
   \centering
   \includegraphics[width=0.7\textwidth]{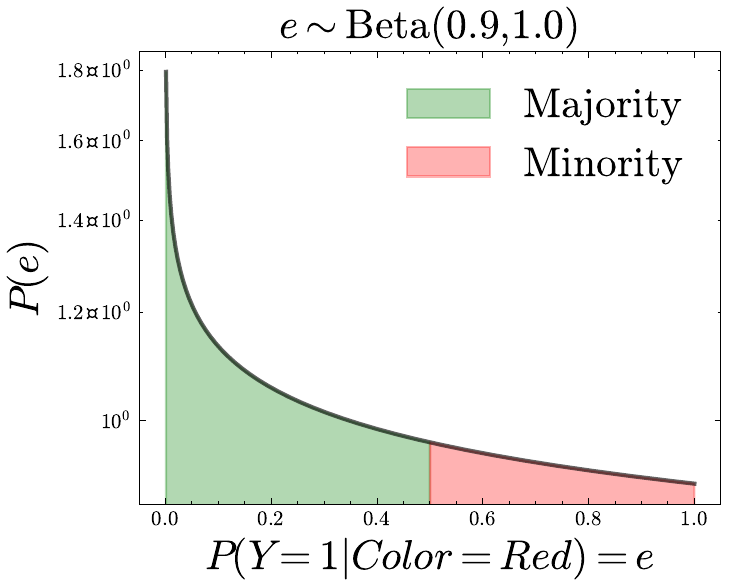}
   \caption{Long tail distribution of train environments}
   \label{fig:dist_of_cmnist_train}
   \end{subfigure}
    \caption{In Figure~\ref{fig:dag_of_cmnist} we describe the features that affect the target. The mechanism by which color affects target changes across environments. However, shape has a stable mechanism across environments. In Figure~\ref{fig:dist_of_cmnist_train} we consider a long tail distribution of environments from which we sample training environments. This is often realistic that many subpopulations are underrepresented in training data, eg low resource languages for translation tasks.}
    \label{fig:exp-setup}
\end{figure*}
This forces all the precise learners with a fixed hypothesis, i.e., \textbf{PL}-$f$ to learn the invariant risk minimizer across domains that rely only on shape as a predictor to generalize to minority domains. We compare performance to baselines (precise learners with fixed hypothesis \textbf{PL}-$f$) based on different assumptions like ERM (average-case risk), GrpDRO~\citep{Sagawa20:DRO}, V-REx~\citep{krueger_out--distribution_2021} (worst-case risk) and IRM~\citep{arjovski2021}, IGA~\citep{koyama2020out} (Invariant Predictors), EQRM~\citep{Eastwood22:QRM} (probable domain generalizer) and SD~\citep{pezeshki2021gradient} which avoids implicit regularization from Gradient starvation by decoupling features. We also consider Inf-Task which is a baseline for comparing how an Imprecise Learner (\textbf{IL}) performs against precise learners with an augmented hypothesis (\textbf{PL}-$\bar{h}$). Based on the initialization setup for CMNIST described by \citet{Eastwood22:QRM}, all baseline methods perform poorly without ERM pretraining. Therefore, to ensure a fair comparison, we consider the ERM pretraining for \textbf{PL}-$f$ learners for the initial 400 steps out of a 600-step training. All other hyper-parameters remain consistent with the established setup. For the learners with augmented hypotheses, it does not make sense to initialize with ERM because it may predispose the imprecise learner towards specific outcomes. Therefore, we assess the best-case performance across all learners across types of initialization. To implement the augmented hypothesis, we append FILM layers ~\citep{perez2018film} to MLP architecture used in ~\citet{Eastwood22:QRM}.
\subsection{Imprecise Learner can learn relevant features in context}
\begin{table}[h]
\centering
\caption{Maximal regret and test accuracy across all CMNIST test environments.\textbf{Bold} denotes the hypothetical best invariant and Bayes classifier performance. Highlighted \textcolor{applegreen}{\textbf{Green}} denotes the best performance amongst all algorithms for each domain and best regret. Bayes classifier is defined w.r.t the IID learner trained for a particular environment}
\resizebox{\textwidth}{!}{
\begin{tabular}{llllllllllllll}
\toprule
\multirow{2}{*}{\bfseries Objective} & \multirow{2}{*}{\bfseries Algorithm} & \multicolumn{11}{c}{Test Environments based on $\PP(Y=1\,|\,\text{color}=\text{red})=e$} & \multirow{2}{*}{\bfseries Regret}\\
 &  & 0.0 & 0.1 & 0.2 & 0.3 & 0.4 & 0.5 & 0.6 & 0.7 & 0.8 & 0.9 & 1.0 & \\
\midrule
Average Case & ERM  & \textcolor{applegreen}{\textbf{96.1}} & 87.1 & 78.0 &  72.1 & 65.8 &  59.2 & 51.8 & 47.1 &  39.9 &  33.6 & 28.3  & 72.7\\
\cline{2-13} 
\multirow{2}{*}{Worse Case} & GrpDRO & 54.1 & 55.6 & 58.1 & 595 &  61.5 &  64.5 &  66.3 & 69.1  & 70.5  & 73.9  &\textcolor{applegreen}{\textbf{75.5}} & 46.9\\
& SD & 52.1 & 54.1 & 56.6 & 58.6 & 59.7 & 63.7 & 65.8 & 67.0 & 68.5 & 70.3 & 73.3 & 47.9\\
\cline{2-13} 
\multirow{5}{*}{Invariance} & IRM  & 72.0 & 72.0 & 72.0& \textcolor{applegreen}{\textbf{72.0}} & \textcolor{applegreen}{\textbf{72.1}} & \textcolor{applegreen}{\textbf{69.7}} & 69.3 & 69.9 & 69.2 & 69.7 & 67.7 & 32.3\\
&IGA & 71.8 & 72.0 & 72.0 & 72.1 & 69.8 & 65.2 & 62.4 & 60.5 & 57.2 & 57.7 & 50.3 & 49.7\\
&EQRM ($\lambda\rightarrow 1$) & 67.8 & 67.7 & 68.3 & 68.8 & 70.5 & 69.1 & 70.3 & \textcolor{applegreen}{\textbf{72.0}} & \textcolor{applegreen}{\textbf{72.1}} & 71.4 & 72.1 & 32.2\\
&VREx & 72.7 & 71.3 & 71.8 & 71.4 & 71.7 & 69.5 & 69.5 & 70.2 & 69.5 & \textcolor{applegreen}{\textbf{71.6}} & 68.5& 31.5 \\
\cline{3-13}
& Oracle & \multicolumn{11}{c}{\textbf{73.5}} & \textbf{27.9}\\
\cline{2-13} 
\textbf{PL}-$\bar{h}$& Inf-Task & 96.0 & 86.3 & 78.6 & 68.0 & 62.1 & 61.3 & 63.2 & 65.0 & 66.6 & 68.4 & 68.3 & 31.7 \\
\cline{2-13} 
\textbf{IL} (Ours) & IRO & 95.8 & 87.2  & 78.8 &  68.9 & 69.4 & 69.5 & \textcolor{applegreen}{\textbf{70.8}} & 70.1 & 70.0 & 70.4 & 70.3 & \textcolor{applegreen}{\textbf{29.7}} \\
\cline{2-13} 
Bayes Classifier& ERM (IID) & \textbf{100.0} & \textbf{90.0} & \textbf{80.0} & \textbf{75.0} & \textbf{75.0} & \textbf{75.0} & \textbf{75.0} & \textbf{75.0} & \textbf{80.0} & \textbf{90.0} & \textbf{100.0} \\
\bottomrule

\end{tabular}}
\label{table:cmnist_regret}
\end{table}
In Table~\ref{table:cmnist_regret} we compare \textbf{IL} to other methods, showing that \textbf{IL} can learn relevant features in context. This also allows us to guide model operators on selecting appropriate $\lambda$. Suppose the user expects data at test time to come from the majority environments of their training. In that case, they can be less risk averse and use $\lambda=0$ whereas if the user is unsure and anticipates test environments to look like unlike training, i.e. more minority environments they can choose $\lambda\rightarrow1$. This is also reflected in the performance of \textbf{IL} such that for the majority domains $e\in\{0.0,\dots,0.4\}$ it performs similar to average case learner and for relatively less seen i.e. minority domains $e\in\{0.5,\dots,1.0\}$ it performs similar to the invariant learner.

\section{Limitations of Imprecise Learner}

\subsection{Computational Complexity}
The additional computation costs result from solving \eqref{eq:discrete-p-stationary-dist} compared to solving for a single notion of generalization which grows by the $O(m)$ where $m$ is the number of estimates needed. Since the convergence rate for Monte Carlo estimates is $O(\frac{1}{\sqrt{m}})$ the quality of estimates of the gradient improves slowly w.r.t. the number of samples. The generalization to the user's choice of risk $\lambda_{\text{op}}$ with high probability is also given by $O(\frac{1}{\sqrt{n}}+\frac{1}{\sqrt{m}})$ in \Cref{prop:regret-bound}, where $n$ is the number of data samples from each environment. In practice, there is room to obtain a better approximation of  ~\eqref{eq:discrete-p-stationary-dist} with possibly quasi-Monte Carlo sampling methods.

\subsection{Challenges in Specifying User Preferences}
\label{sec:test-time-elicit}

One of the main challenges in the Imprecise Learning (\textbf{IL}) framework is to specify user preference in terms of risk level i.e. a choice of $\lambda_{\text{op}}$. In practical scenarios, model operators may encounter challenges in precisely articulating their level of risk aversion. Additionally, bridging the operator’s concept of generalization to a specific domain with an appropriate risk level remains ambiguous. In our experiments on modified CMNIST, we address this by allowing the model operator to be more risk-averse to generalize to minority environments. In contrast, for generalizing to a domain from majority environments users can be more risk-seeking.

\subsection{Generalization with no access to minority environments}
In the context of the standard CMNIST setup where the learner has access to no minority environments, CVaR as a risk measure does not allow to generalize beyond the credal set which can be constructed from the convex combination of majority environments alone. For standard CMNIST setup training envs are $\{0.1,0.2\}$ and test env is $\{0.9\}$. This means that the mechanism by which color affects the target is anti-correlated at test time, such situations can arise in adversarial settings. Since for $\lambda\rightarrow 1$, CVaR only minimizes the higher risks in a profile to achieve invariance it cannot recover the invariant mechanism without access to at least one environment from a subgroup. However, we argue that by using additional assumptions i.e. a different risk measure Imprecise learners can still learn to generalize to novel unseen domains outside of the credal set. We can extend the risk measure to enforce invariance by using VREx as an additional regularizer. 
\begin{equation}
    \rho_\lambda[\bm{\cR}] := \text{CVaR}_\lambda[\bm{\cR}]+\lambda \text{Variance}(
    \bm{\cR})
\end{equation}

In \Cref{table:cmnist_test_accuracy}, we observe that \textbf{IL} for $\lambda=1$ obtains poor performance on a novel test domain however with an additional risk measure it obtains a closer performance to ERM on grayscale (Oracle) and outperforms several baselines. Note that with random initialization \textbf{IL}+VREx significantly outperforms other baselines.

\begin{table}[ht]
\centering
\caption{CMNIST Test Accuracy. Training Environments are $\{0.1,0.2\}$ \& Test Environment $\{0.9\}$}
\begin{tabular}{llccc}
\hline
\multirow{2}{*}{\textbf{Objective}} & \multirow{2}{*}{\textbf{Algorithm}} & \multicolumn{3}{c}{\bfseries Initialization} \\
& & Rand. & ERM & Best Case\\
\hline
\multirow{7}{*}{\textbf{PL-}$f$}& ERM & 27.9 $\pm$ 1.5 & 27.9 $\pm$ 1.5 & 27.9 $\pm$ 1.5\\
& IRM & 52.5 $\pm$ 2.4 & 69.7 $\pm$ 0.9 & 69.7 $\pm$ 0.9 \\
& GrpDRO & 27.3 $\pm$ 0.9 & 29.0 $\pm$ 1.1 & 29.0 $\pm$ 1.1\\
& SD & 49.4 $\pm$ 1.5 & 70.3 $\pm$ 0.6 & 70.3 $\pm$ 0.6 \\
& IGA & 50.7 $\pm$ 1.4 & 57.7 $\pm$ 3.3 & 57.7 $\pm$ 3.3 \\
& V-REx & 55.2 $\pm$ 4.0 & \textbf{71.6} $\pm$ \textbf{0.5} & \textbf{71.6} $\pm$ \textbf{0.5}\\
& EQRM & 53.4 $\pm$ 1.7 & 71.4 $\pm$ 0.4 & 71.4 $\pm$ 0.4 \\
\hline 
\textbf{IL} & IRO & 28.4 $\pm$ 0.7 &  27.4 $\pm$ 0.1 & 28.4 $\pm$ 0.7  \\
\textbf{PL-}$\bar{h}$+VREX & Inf-Task & 68.4 $\pm$ 0.1 & 64.6 $\pm$ 0.0 & 68.4 $\pm$ 0.1 \\
\textbf{IL}+VREX & IRO & \textbf{71.4} $\pm$ \textbf{0.2} &  65.4 $\pm$ 0.1 & 71.4 $\pm$ 0.2  \\
\hline 
Invariant Pred. & Oracle & \multicolumn{3}{c}{\textbf{73.5} $\pm$ \textbf{0.2}} \\
\hline
\end{tabular}
\label{table:cmnist_test_accuracy}
\end{table}


\section{Implementation Details}

This section provides the details of specific implementations used in our experiments.

\subsection{Augmented Hypothesis}

For implementing the augmented hypothesis, we use hypernetworks~\cite{ha2016hypernetworks} to realize the dependence of $h$ on model operator's preference, i.e., $\lambda$. In this scenario, the weights of the augmented model are dependent on $\lambda$, i.e., $h_\xi(x,\lambda):=f_{g_{w}(\lambda)}(x)$ where $g_w(\lambda)$ is the hypernetwork and $\xi:=\{w,g_w(\lambda)\}$.  For neural networks with multiple layers, we use FILM layers~\cite{perez2018film} to augment the network such that it can be conditioned upon $\lambda$.

\subsection{Imprecise Risk Optimisation}

To operationalise the imprecise risk optimization, we need to minimise \eqref{eq:discrete-p-stationary-dist} with respect to the family of probability distributions $\Delta(\Lambda)$. Since for our case $\Lambda=[0,1]$, we parameterise the family of distributions with $\text{Beta}(\alpha,\beta)$. We sample $\lambda$ from $Q$ via uniform sampling from the inverse CDF of $Q$, which we denote as $F^{-1}$. We approximate the gradient of $F^{-1}$ by first-order difference as described in \Cref{alg:icdf_beta_impl}.
\begin{algorithm}[h]
    \caption{Sampling from a Beta Distribution using ICDF with Gradient Computation}
    \label{alg:beta-sampling-grad}
    \begin{algorithmic}[1]
        \STATE \textbf{class} ICDFBeta:
        \STATE \quad \textbf{def} forward($u$): \# Compute ICDF
        \STATE \quad\quad \textbf{return} F$^{-1}$($\alpha$,$\beta$)($u$) 
        \STATE \quad \textbf{def} backward($u$): \# Compute Gradient
        \STATE \quad\quad $\delta:= 1e-6$
        \STATE \quad\quad $\nabla_\theta F^{-1}(\alpha,\beta)(u):=$ $\frac{(F^{-1}(\alpha+\delta, \beta)(u) - F^{-1}(\alpha, \beta)(u))}{\delta}$ 
        \STATE \quad\quad $\nabla_\phi F^{-1}(\alpha,\beta)(u):=$ $\frac{(F^{-1}(\alpha, \beta+\delta)(u) - F^{-1}(\alpha, \beta)(u))}{\delta}$ 
        \STATE \quad\quad \textbf{return} $\nabla_\alpha F^{-1}(\alpha,\beta)(u)$, $\nabla_\beta F^{-1}(\alpha,\beta)(u)$

        \STATE \textbf{Initialize:} $\alpha, \beta \gets 1.0, 1.0$
        \STATE icdfbeta = ICDFBeta($\alpha$, $\beta$)
        \FOR{epoch $= 1$ \textbf{to} $k$}
            \FOR{$i = 1$ \textbf{to} $m$}
                \STATE $u_i \sim \operatorname{Uniform}([0, 1])$
                \STATE $\lambda_i$ = icdfbeta.forward($u_i$)
            \ENDFOR
        \ENDFOR
        \STATE \textbf{Return} Set of samples $\{\lambda_1, \lambda_2, \ldots, \lambda_m\}$ and gradients for each epoch
    \end{algorithmic}
    \label{alg:icdf_beta_impl}
\end{algorithm}

\end{document}